\documentclass{article}

\usepackage[final]{neurips_2022}

\usepackage[utf8]{inputenc} % allow utf-8 input
\usepackage[T1]{fontenc}    % use 8-bit T1 fonts
\usepackage{hyperref}       % hyperlinks
\usepackage{url}            % simple URL typesetting
\usepackage{booktabs}       % professional-quality tables
\usepackage{amsfonts}       % blackboard math symbols
\usepackage{nicefrac}       % compact symbols for 1/2, etc.
\usepackage{microtype}      % microtypography
\usepackage{xcolor}         % colors

\usepackage{url}            % simple URL typesetting
\usepackage{booktabs}       % professional-quality tables
\usepackage{amsfonts}       % blackboard math symbols
\usepackage{nicefrac}       % compact 
\usepackage{microtype}      %
\usepackage{framed}
\usepackage{amssymb}
\usepackage{amsfonts}
\usepackage{mathrsfs}
\usepackage{changepage}
\usepackage{mathtools}
\usepackage{array}
\usepackage{amsthm}
\usepackage{verbatim} 
\usepackage{enumerate}
\usepackage{bbm}
\usepackage{commath}
\usepackage{wrapfig}
\usepackage{amsbsy}
\usepackage{float}
\usepackage{amsmath}
\usepackage{tabularx}
\usepackage{listings}
\usepackage{enumitem}
\usepackage{lipsum}
\usepackage{algpseudocode}
\usepackage[ruled,vlined]{algorithm2e}

\definecolor{darkblue}{rgb}{0.0,0.0,0.65}
\definecolor{darkred}{rgb}{0.68,0.05,0.0}
\definecolor{darkgreen}{rgb}{0.0,0.29,0.29}
\definecolor{darkpurple}{rgb}{0.47,0.09,0.29}
\hypersetup{
   colorlinks = true,
   citecolor  = darkblue,
   linkcolor  = darkred,
   filecolor  = darkblue,
   urlcolor   = darkblue,
 }

\newcommand{\AL}[1]{\textcolor{blue}{\textbf{AL:}{ #1}}}
\newcommand{\JR}[1]{\textcolor{purple}{\textbf{JR:}{ #1}}}

\newcommand{\calX}{\mathcal{X}}
\newcommand{\xb}{\mathbf{x}}
\newcommand{\zb}{\mathbf{z}}

\newcommand{\barf}{\mathfrak{F}} %{\bar{f}}

\renewcommand{\emptyset}{\varnothing}

\theoremstyle{definition}

\newtheorem{proposition}{Proposition}
\newtheorem{lemma}{Lemma}

\newtheorem{corollary}{Corollary}
\newtheorem*{definition}{Definition}

\newtheorem*{proposition*}{Proposition}

\usepackage{amsmath,amsfonts,bm}

\def\eqref#1{equation~\ref{#1}}
\def\1{\bm{1}}

\def\vone{{\bm{1}}}

\DeclareMathAlphabet{\mathsfit}{\encodingdefault}{\sfdefault}{m}{sl}
\SetMathAlphabet{\mathsfit}{bold}{\encodingdefault}{\sfdefault}{bx}{n}

\newcommand{\E}{\mathbb{E}}

\DeclareMathOperator*{\argmin}{arg\,min}

\title{Neural Set Function Extensions: \\ Learning with Discrete Functions in High Dimensions}

\author{%
  Nikolaos Karalias$^*$ \\
  EPFL\\
  \texttt{nikolaos.karalias@epfl.ch} \\
  \And
  Joshua Robinson\thanks{Equal contribution.} \\
   MIT CSAIL \\
   \texttt{joshrob@mit.edu} \\
     \And
  Andreas Loukas \\
  Prescient Design, Genentech, Roche\\
   \texttt{andreas.loukas@roche.com} \\
    \And
  Stefanie Jegelka \\
   MIT CSAIL \\
   \texttt{stefje@csail.mit.edu} \\
}

\begin{document}

\maketitle

\begin{abstract}

Integrating functions on discrete domains into neural networks is key to developing their capability to reason about discrete objects. But, discrete domains are (I) not naturally amenable to gradient-based optimization, and (II) incompatible with deep learning architectures that rely on representations in high-dimensional vector spaces. In this work, we address both difficulties for set functions, which capture many important discrete problems.
First, we develop a framework for extending set functions onto low-dimensional continuous domains, where many extensions are naturally defined. Our framework subsumes many well-known extensions as special cases. Second, to avoid undesirable low-dimensional neural network bottlenecks, we convert low-dimensional extensions into representations in high-dimensional spaces, taking inspiration from the success of semidefinite programs for combinatorial optimization. Empirically, we observe benefits of our extensions for unsupervised neural combinatorial optimization, in particular with high-dimensional representations.
\end{abstract}
%%%%%%%%%%%%%%%%%%%%%%%%%%%%%%%%%%%%%%%%%%%%%%%%%%%%%%%%%%%%%%%%%%%%%%%%%%%%
\section{Introduction}

While neural networks are highly effective at solving tasks grounded in basic perception \citep{chen2020simple,vaswani2017attention}, discrete algorithmic and combinatorial tasks such as partitioning graphs, and finding optimal routes or shortest paths have proven more challenging. This is, in part, due to the difficulty of integrating discrete operations into neural network architectures \citep{battaglia2018relational,bengio2021machine,cappart2021combinatorial}. One immediate difficulty with functions on discrete spaces is that they are not amenable to standard gradient-based training. Another is that discrete functions  are typically expressed in terms of scalar (e.g., Boolean) variables for each item (e.g., node, edge to be selected), in contrast to the high-dimensional and continuous nature of neural networks' internal representations. A natural approach to addressing these challenges is to carefully choose a function on a continuous domain that \emph{extends} the discrete function, and can be used as a drop-in replacement.

 There are several important desiderata that such an extension should satisfy in order to be suited to neural network training. First, an extension should be valid,
 %\sj{is ``well-posed'' and official term? Maybe it should be an exact interpolation?} 
 i.e., agree with the discrete function on discrete points. It should also be amenable to gradient-based optimization, and should avoid introducing spurious minima. Beyond these requirements, there is one additional critical consideration. In both machine learning and optimization, it has been observed that high-dimensional representations can make problems ``easier''. For instance, neural networks rely on high-dimensional internal representations for representational power and to allow information to flow through gradients, and performance suffers considerably when undesirable low-dimensional bottlenecks are introduced into network architectures \citep{belkin2019reconciling,VELICKOVIC2021100273}. In optimization, \emph{lifting} to higher-dimensional spaces can make the problem more well-behaved \citep{goemans1995improved,shawe2004kernel,du2018gradient}. Therefore,  extending discrete functions to \emph{high-dimensional} domains may be critical to the effectiveness of the resulting learning process, yet remains largely an open problem.

With those considerations in mind, we propose a framework for constructing extensions of discrete set functions onto high-dimensional continuous spaces.
  The core idea is to view a continuous point $\xb$ in space as an expectation over a distribution (that depends on $\xb$) supported on a few carefully chosen discrete points, to retain tractability. To evaluate the discrete  function at $\xb$, we compute the expected value of the set function over this distribution. The method resulting from a principled formalization of this idea is computationally efficient and addresses the key challenges of building continuous extensions. Namely, our extensions allow gradient-based optimization and address the dimensionality concerns,  allowing any function on sets to be used as a computation step in a neural network.  

First, to enable gradient computations, we present a method based on a linear programming (LP) relaxation for constructing extensions on continuous domains where exact gradients can be computed using standard automatic differentiation software \citep{abadi2016tensorflow,bastien2012theano,paszke2019pytorch}. Our approach allows task-specific considerations  (e.g., a cardinalilty constraint) to be built into the extension design. While our initial LP formulation handles gradients, and is a natural formulation for explicitly building extensions, it replaces discrete Booleans with  scalars in the unit interval $[0,1]$, and hence does not yet address potential dimensionality bottlenecks. Second, to enable higher-dimensional representations, we take inspiration from classical SDP relaxations, such as  the celebrated Goemans-Williamson maximum cut algorithm \citep{goemans1995improved}, which recast low-dimensional problems in high-dimensions. Specifically, our key contribution is to develop an SDP analog of our original LP formulation, and show how to \emph{lift} LP-based extensions into a corresponding high-dimensional SDP-based extensions. Our general procedure for lifting low-dimensional representations into higher dimensions aligns with the neural algorithmic reasoning blueprint \citep{VELICKOVIC2021100273}, and suggests that classical techniques such as SDPs may be effective tools for combining deep learning with algorithmic processes more generally. 

\section{Problem Setup}
\label{sec: method}
Consider a ground set $[n]=\{1,\ldots , n\}$ and an arbitrary function $f : 2^{[n]} \rightarrow \mathbb{R} \cup \{\infty\}$ defined on subsets of $[n]$. For instance, $f$ could determine
if a set of nodes or edges in a graph has some structural
property, such as being a path, tree, clique, or independent set \citep{bello2016neural,cappart2021combinatorial}. Our aim is to build neural networks that use such discrete functions $f$ as an intermediate layer or loss. In order to produce a model that is trainable using standard auto-differentiation software, we consider a continuous domain $\calX$ onto which we would like to extend $f$, with sets  embedded into $\calX$ via an injective map $e : 2^{[n]} \rightarrow \calX$. For instance, when  $\calX=[0,1]^n$ we may take $e(S) = \mathbf{1}_S$, the Boolean vector whose $i$th entry is $1$ if $i \in S$, and $0$ otherwise. Our approach is to design an extension 
\begin{align*}
    \barf : \calX \rightarrow  \mathbb{R}
\end{align*}
of $f$ and consider the neural network $\text{NN}_2 \circ \barf \circ\text{NN}_1$ (if $f$ is used as a loss, $\text{NN}_2$ is simply the identity).
To ensure that the extension is \textit{valid} and amenable to automatic differentiation, we require that 1) it agrees with $f$ on all discrete points: 
 $\barf(e(S)) = f(S) \  \text{for all} \  S \subseteq [n] \  \text{with} \  f(S) < \infty $, and 2) $\barf$ is continuous. 

There is a rich existing literature on extensions of functions on discrete domains, particularly in the context of discrete optimization \citep{lovasz1983submodular,gls81,caliChVon11,vondrak08, bach2019submodular,obozinski2012convex,tawarmalani2002convex}. These works provide promising tools to reach our goal of neural network training. 
Building on these, our method is the first to use semi-definite programming (SDP) to combine neural networks with set functions.  There are, however, different considerations in the neural network setting as compared to optimization.  The optimization literature often focuses on a class of set functions and aims to build  extensions with desirable optimization properties, particularly convexity. We do not focus on convexity, aiming instead to develop a formalism that is as flexible as possible. Doing so maximizes the applicability of our method, 
%within which many possible valid extensions may be instantiated. This flexibility allows to adapt both the set function and its extension to be
and allows extensions adapted to task-specific desiderata (see Section \ref{sec: constructing scalar SFEs}).

\section{Scalar Set Function Extensions}\label{sec: scalar SFEs}
%\AL{The scalar might be confusing as we are embedding sets in $\mathbb{R}^n$. I understand why you call it this way but I maintain some doubts.}

We start by presenting a general framework for extending set functions onto $\calX = [0,1]^n$, where a set $S \subseteq [n]$ is viewed as the Boolean indicator vector  $e(S) = \mathbf{1}_S \in \{0,1\}^n$ whose $i$th entry is $1$ if $i \in S$ and $0$ otherwise. We call extensions onto $[0,1]^n$ \emph{scalar} since each item $i$ is represented by a single scalar value---the $i$th coordinate of $\xb \in \calX$. These scalar extensions will become the core building blocks in developing high-dimensional extensions in Section \ref{sec: neural SFEs}.

% Sets $S \subseteq [n]$ are naturally represented by Boolean indicator vectors $\mathbf{1}_S \in \{0,1\}^n$ whose $i$th entry is $1$ if $i \in S$ and $0$ otherwise. This representation suggests the hypercube $\calX = [0,1]^n$ as a natural continuous domain onto which to extend set functions. % However, since each item is represented by a single scalar in $[0,1]$, this choice of domain introduces dimensionality bottlenecks. section presents a general framework for extending set functions onto $\calX = [0,1]^n$. These low-dimensional extensions will become key building blocks in the development of high-dimensional extensions presented in Section \ref{sec: neural SFEs}.

  \begin{figure*}[t] %{6.5cm}
   \begin{center}
     \includegraphics[width=0.8\textwidth]{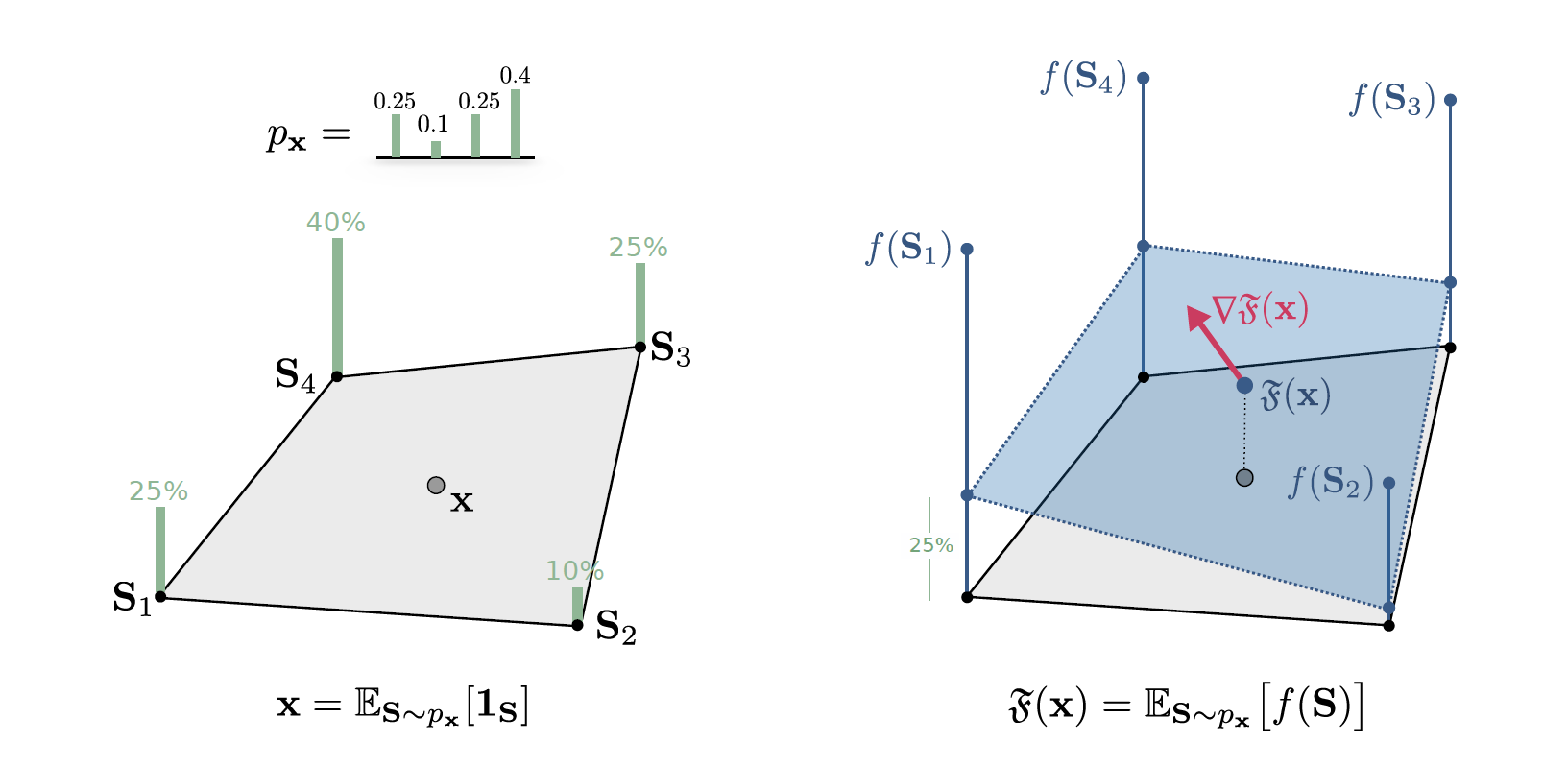}
       \vspace{-10pt}
       \parbox{140mm}{
     \caption{\textbf{SFEs:} Fractional points $\xb$ are reinterpreted as expectations $\xb = \mathbb{E}_{S \sim p_\xb} [\mathbf{1}_S]$ over the distribution $p_\xb(S)$ on sets. A value  is assigned at $\xb$ by exchanging the order of $f$ and the expectation: $\barf (\xb)_{S \sim p_\xb}[f(S)]$. Unlike $f$, the extension $\barf$ is amenable to gradient-based optimization.}
     \label{fig:idea}
     }  
   \end{center}
 \end{figure*}
 
A classical approach to extending discrete functions on sets represented as Boolean indicator vectors $\mathbf{1}_S$ is by 
%1) representing sets by Boolean indicator vectors $e(S) = \mathbf{1}_S \in \{0,1\}^n$ whose $i$th entry is $1$ if $i \in S$ and $0$ otherwise and 2) 
computing the convex-envelope, i.e., the point-wise supremum over linear functions that lower bound $f$ \citep{falk1976successive,bach2019submodular}.
%}. Since computing the convex envelope is typically intractable (except in certain special cases), it is common to restrict the supremum to be over \emph{linear} functions \citep{bach2019submodular}.\sj{shouldn't this be the same, by continuity? where would you not do this? I would just do it with linear, and remove this discussion. The crux is that a point-wise max of cvx functions is still convex}  
Doing so yields a convex function whose value at a point $\mathbf{x} \in [0,1]^n$ is the solution of the following linear program (LP):
%
%In accordance with the definition of the convex closure, we formulate the following LP relaxation that seeks for a point $\mathbf{x} \in \calX = \mathbb{R}^n_+$ the point-wise maximum over functions that lower-bound $f$:
%
\begin{align*}
    \widetilde{\barf}(\xb) = \underset{\mathbf{z},b \in \mathbb{R}^n \times \mathbb{R}}{\max}\{ \xb^{\top}\zb +b\} 
    \ \text{ subject to } \  \mathbf{1}_S ^\top \mathbf{z} +b  \leq f(S) \  \text{ for all } S \subseteq [n].   \tag{primal LP}
\end{align*}
% 
%\AL{Why do we use min/max here and sup/inf in the SDP?}
The set $\mathcal{P}_f$ of all feasible solutions $(\mathbf{z},b)$ is known as the \emph{(canonical) polyhedron of $f$} \citep{obozinski2012convex} and can be seen to be non-empty by taking the coordinates of $\zb$ to be sufficiently small (possibly negative). Variants of this optimization program are frequently encountered in the theory of matroids and submodular functions \citep{edmonds2003submodular} where $\mathcal{P}_f$ is commonly known as the \emph{submodular polyhedron} (see Appendix \ref{app: primaldual} for an extended discussion). 
%The solution to the LP lower bounds the convex envelope of $f$, and is much more tractable. 
By strong duality, we may solve the primal LP by instead solving its dual:
%\citep[Def. 20]{el2018learning} \sj{do we need a ref for that?}
% 
\begin{align}
    \widetilde{\barf}(\xb) = \underset{\{y_S\geq 0\}_{S \subseteq [n]}}{\min} \sum_{S \subseteq [n]} y_S f(S) \nonumber
    \text{ subject to } \sum_{S \subseteq [n]} y_S \mathbf{1}_S = \xb, \  \sum_{S \subseteq [n]} y_S=1, \ \text{ for all } S\subseteq [n], \tag{dual LP}
\end{align} \nonumber
whose optimal value is the same as the primal LP. The dual LP is always feasible (see e.g., the Lov\'asz extension in Section \ref{sec: constructing scalar SFEs}). However,  $\widetilde{\barf}$ does not necessarily agree with $f$ on discrete points in general, unless the function is convex-extensible \citep{murota1998discrete}.
% : the primal LP provides a valid extension only when the optimal $\zb$ makes the inequality constraints  tight for all $\mathbf{x} = \mathbf{1}_S$ with $S \subseteq [n]$. 
%\JR{(sightly confused by this sentence)}\sj{do you mean for all optimal solutions $z$ when $x$ is binary, the inequality in the primal must be tight for that set? It's indeed vague}

To address this important missing piece, we relax our goal from solving the dual LP to instead seeking a \emph{feasible} solution to the dual LP that \emph{is} an extension of $f$.  Since the dual LP is defined for a fixed $\xb$, a feasible solution must be a function $y_S = p_\xb (S)$ of $\xb$. If $p_\xb$ were to be continuous and a.e. differentiable in $\xb$ then the value $\sum_{S} p_\xb (S) f(S)$ attained by the dual LP  would also be continuous and a.e. differentiable in $\xb$ since gradients flow through the coefficients $y_S= p_\xb (S)$, while $f(S)$ is treated as a constant in $\xb$. This leads us to the following definition:
\begin{definition}[Scalar SFE]
A scalar SFE $\barf$ of $f$ is defined at a point $\xb \in [0,1]^n$ by coefficients $p_\xb(S)$ such that  $y_S=p_\xb (S)$ is a feasible solution to the dual LP. The extension value is given by $$\barf (\xb) \triangleq \sum_{S \subseteq [n]} p_\xb (S) f(S)$$ and we require the following properties to hold for all $S \subseteq [n]$: 1) $p_\xb(S)$ is a continuous function of $\xb$ and 2) $\barf(\mathbf{1}_S)=f(S)$ for all $S \subseteq [n]$. 
    % \begin{enumerate}
    %     \item (Differentiable) $p_\xb(S)$ is a continuous and a.e. differentiable function of $\xb$,
    %     \item (Extension)  $\barf(\mathbf{1}_S)=f(S)$.
    % \end{enumerate}  
\end{definition}
%
%\AL{shall we add name in the subscript of $\barf$?}
Efficient evaluation of $\barf$ requires that $p_\xb(S)$ is supported  on a small collection of carefully chosen sets $S$. This choice is a key inductive bias of the extension, and  Section \ref{sec: constructing scalar SFEs} gives many examples with only $ O(n)$ non-zero coefficients. 
% remains to construct interesting instances of coefficients $p_\xb(S)$ that satisfy the definition of a scalar SFE. Fortunately it is possible to construct scalar SFEs, as we show in 
%Section \ref{sec: constructing scalar SFEs} presents several  examples of scalar SFEs. 
Examples include well-known extensions, such as the Lov\'asz extension, as well as a number of novel extensions, illustrating the versatility of the SFE framework. %, and demonstrates how prior knowledge about $f$ (e.g., a cardinality constraint) can be built into extensions.
%\JR{Say something about who the choice  to  only require feasibility of dual LP also allows freedom to construct different SFE, which have different characteristics.}\sj{In a sense, a special case of this is to learn a distribution over sets. How does that relate to the literature? How are we more general (e.g., ours is not actually a distribution by the constraints)}

Thanks to the constraint $\sum_S y_S =1$ in the dual LP, scalar SFEs have a natural probabilistic interpretation. An SFE is defined by a probability distribution $p_\xb$ such that fractional points $\xb$ can be written as an expectation $\mathbb{E}_{S \sim p_\xb} [\mathbf{1}_S] = \xb$ over discrete points using $p_\xb$. The extension itself can be viewed as arising from exchanging $f$ and the expectation operation: $\barf(\xb) = \mathbb{E}_{S \sim p_\xb} [f(S)]$. This interpretation is summarized in Figure \ref{fig:idea}.

Scalar SFEs also enjoy the property of not introducing any spurious minima. That is, the minima of $\barf$ coincide with the minima of $f$ up to convex combinations. This property is especially important when training models of the form $f \circ \text{NN}_1$ (i.e., $f$ is a loss function) since $\barf$ will guide the network $\text{NN}_1$ towards the same solutions as $f$. 

%An important class of SFEs, which have additional desirable properties, are probabilistic SFEs: scalar SFEs whose coefficients also sum to one: $\sum_S g(\xb ;S) =1$. In this case, the extension $\barf$ can be thought of as an expectation $\barf(\xb) = \mathbb{E}_{S \sim g} f(S)$ where $g(\xb ;S)$ is the probability of sampling set $S$. \JR{say something about Maddison et al. work here since probabilistic form looks v. similar to theirs, but  is (of course) different}. Probabilistic SFEs are guaranteed to not introduce additional spurious  minima. 
%\SJ{for the probabilistic version, do we still assume the constraint that $\sum_S y_S 1_S = \xb$?}
%
\begin{proposition}[Scalar SFEs have no bad minima]\label{prop: nice properties of extension}
If $\barf$ is a scalar SFE of $f$ then:
\begin{enumerate}
    \item $\min_{\xb\in  \calX} \barf(\xb)=\min_{S \subseteq [n]} f(S) $
    \item %The minima of $ \barf(\xb)$ over $\xb\in [0,1]^d$ are a subset of the convex hull of  the minima of $ f(\xb)$ over $\xb\in \Omega$ 
    $\argmin_{\xb\in \calX} \barf(\xb) \subseteq \text{Hull} \big (\argmin_{\mathbf{1}_S : S \subseteq [n]} f(S) \big )$ 
\end{enumerate}
\end{proposition}
See Appendix \ref{app: vector SFE proofs} for proofs. 
%\sj{In a sense, the probabilistic version is the LP relaxation of the ``stupid'' or generic IP $\min_{\mathbf{y}} \sum_{S} y_S f(S)$, $\sum_S y_S = 1, y_S \geq 0$ and integer. Update: actually this is only really true without the level set constraint. The level set constraint gives a different, more restricted relaxation and compact encoding.}
%Beyond the geometric properties of scalar SFEs, there are also computational considerations, discussed next.  

%\textbf{Complexity of pointwise evaluation of  extension.} Since in general $\barf$ is a sum of $2^n$ terms, efficient evaluation of $\barf$ requires that there are only a few $S$ for which $p_\xb(S)$ is non-zero. In fact, carefully choosing a (small) collection of sets to give non-zero probability is a key inductive bias in the extension design and most SFEs introduced in Section \ref{sec: constructing scalar SFEs} have $\mathcal O(n)$ non-zero coefficients. All our experiments use scalar SFEs that can be evaluated efficiently. 

\textbf{Obtaining set solutions.} Given an architecture $\barf \circ\text{NN}_1$ and input problem instance $G$, we often wish to produce sets as outputs at inference time. To do this, we simply compute $\xb = \text{NN}_1(G)$, and select the set $S$ in $\text{supp}_S\{  p_\xb ( S) \}$  with the smallest value $f(S)$. This can be done efficiently if, as is typically the case, the cardinality of $\text{supp}_S\{  p_\xb ( S) \}$ is small.

%\AL{The point is lost in the example.  Start with the main statement and then exemplify it after. } At inference time we often wish to produce sets as outputs. For instance, suppose we wish to train a neural network max-cut solver and let $f$ be the cut size function with extension $\barf$. A natural objective is to  train $\text{NN}_\theta$ to solve $\max_
%\theta\mathbb{E}_G [\barf \circ \text{NN}_\theta(G)]$. For inference on a new graph $G'$, we compute the embedding $\xb =  \text{NN}_\theta(G')$ and consider the sets with non-zero coefficients $\mathcal S = \text{supp}_S\{  p_\xb ( S) \}$. In the typical case where the cardinality of $\mathcal S$ is small---e.g., $\mathcal O(n)$---we may simply select  $S^* \in \arg\min_{S \in \mathcal S}f(S)$ as our set solution. 

\subsection{Constructing Scalar Set Function Extensions}\label{sec: constructing scalar SFEs}

A key characteristic of scalar SFEs is that there are many potential extensions of any given $f$.  In this section, we provide examples of scalar SFEs, illustrating the capacity of the SFE framework for building knowledge about $f$ into the extension. See Appendix \ref{app: vector SFE examples} for all proofs and further discussion. 

\textbf{Lov\'asz extension.}
Re-indexing the coordinates of $\xb$ so that $x_1 \geq x_2 \ldots \geq x_n$, we define $p_\xb$ to be supported  on the sets $ S_1 \subseteq S_2 \subseteq \dots \subseteq S_n$ with $S_i=\{1,2,\dots,i \}$ for $i=1,2,\dots,n$. The coefficient are defined as $y_{S_i} = p_\xb( S_i) := x_i - x_{i+1}$ and $p_\xb( S)=0$ for all other sets. The resulting \emph{Lov\'asz extension}---known as the \emph{Choquet integral} in decision theory \citep{choquet1954theory,marichal2000axiomatic}---is a key tool in combinatorial optimization due to a seminal result: the Lov\'asz extension is convex if and only if $f$ is submodular \citep{lovasz1983submodular}, implying that submodular minimization can be solved in polynomial-time \citep{gls81}.

\textbf{Bounded cardinality Lov\'asz extension.}  A collection $\{S_i\}_{i=1}^n$ of subsets of $[n]$ can be encoded in an $n \times n$ matrix $\mathbf{S} \in \{0,1\}^{n \times n}$ whose $i$th column is  $\mathbf{1}_{S_i}$. In this notation, the dual LP constraint  $\sum_{S \subseteq [n]} y_S \mathbf{1}_S = \xb$ can be written as $\mathbf{S} \mathbf{p} = \xb$, where the $i$th coordinate of  $\mathbf{p}$ defines $p_\xb(S_i)$.
The \emph{bounded cardinality} extension generalizes  the Lov\'asz extension to focus only on sets of cardinality at most $k\leq n$. Again, re-index $\xb$ so that $x_1 \geq x_2 \ldots \geq x_n$. Use the first $k$ sets $S_1 \subseteq S_2 \subseteq \dots \subseteq S_k$, where $S_i=\{1,2,\dots,i \}$, to populate the first $k$ columns of matrix $\mathbf{S}$. 
%As with the Lov\'asz extension, our distribution $p_\xb$ will be supported on exactly $n$ sets. 
We add further $n-k$ sets: $S_{k+i}=\{j+i \; | \; j \in S_k\}$ for $i=1,\dots,n-k$, to fill the rest of $\mathbf{S}$. 
%\mathbf{S}$ is a Toeplitz banded upper triangular matrix and its inverse is itself a triangular Toeplitz matrix \citep{trench1974inversion,meekband}. 
Finally, $p_\xb (S_i)$ can be analytically calculated from $\mathbf{p}= \mathbf{S}^{-1}\mathbf{x}$, where  $\mathbf{S}$ is invertible since it is a Toeplitz banded upper triangular matrix.  
%Its coordinates are $p_\xb (S_i)= \sum_{j \in T_{i,k}} (x_j-x_{j+1})$ where  $T_{i,k}= \{j \; | \; (j-i) \text{ mod } k =0, \text{ for } i\leq j\leq n, \; j \in \mathbb{Z}_+\}$, i.e., $T_{i,k}$ stores the indices where $j-i$ is perfectly divided by $k$.   In the appendix, we show that the bounded cardinality extension is a feasible solution to the dual LP and coincides with the Lov\'asz extension when $k=n$. 

\textbf{Permutations and involutory extensions.} We use the same $\mathbf{S}, \mathbf{p}$ notation. Let $\mathbf{S}$ be an elementary permutation matrix. Then it is involutory, i.e., $\mathbf{S}\mathbf{S}=\mathbf{I}$,
  and we may easily determine $\mathbf{p}=\mathbf{S} \mathbf{x}$ given $\mathbf{S} $  and $\xb$. Note that $p_\xb(S_i)=\mathbf{p}_i$ must  be non-negative since $\xb$ and $\mathbf{S}$ are  non-negative entry-wise.  Finally, restricting $\xb$ to the $n$-dimensional Simplex guarantees that $\|\mathbf{p}\|_1 \leq 1$, which ensures $p_\xb$ is a probability distribution (any remaining mass is placed on the empty set). The extension property can be guaranteed on singleton sets as long as the chosen permutation admits a fixed point at the argmax of $\mathbf{x}$. Any elementary permutation matrix $\mathbf{S}$ with such a fixed point yields a valid SFE.  

% Furthermore, suppose that $\mathbf{S}$ is a permutation matrix. Since it is known that a non-negative matrix has a non-negative inverse if and only if it is a generalized permutation matrix \citep{minc1974nonnegative} we may again be certain that $\mathbf{p}=\mathbf{S}^{-1} \mathbf{S} \mathbf{p} = \mathbf{S} ^{-1} \xb$ has non-negative coordinates. Therefo

\textbf{Singleton extension.} Consider a set function $f$ for which $f(S)=\infty$ unless $S$ has cardinality one. To ensure $\barf$ is finite valued, $p_\xb$ must be supported only on the sets $S_i=\{i\}$, $i=1,\ldots,n$. Assuming $\xb$ is sorted so that $x_1 \geq x_2 \ldots \geq x_n$, define $p_\xb( S_i) = x_i - x_{i+1}$. It is shown in Appendix \ref{app: vector SFE examples} that this defines a scalar SFE, except for the dual LP feasibility. However, when using $\barf$ as a loss function, minimization drives $\xb$ towards the minima $\min_\xb \barf(\xb)$ which \emph{are} dual feasible. So dual infeasibility is benign in this instance and we approach the feasible set from the outside.

%Therefore, the columns of any $n \times n$ permutation matrix can be used in the set expansion. Furthermore, recall that $\mathbf{S}$ is an involutory matrix when $\mathbf{S}\mathbf{S}=\mathbf{I}$. This implies that any binary involutory matrix $\mathbf{S}$ yields a valid SFE.  

\textbf{Multilinear extension.} The multilinear extension, widely used in combinatorial optimization \citep{caliChVon11}, is supported on all sets with coefficients $p_\xb(S)= \prod_{i\in S }x_i \prod_{i\notin S }(1-x_i)$, the product distribution. In general, evaluating the multilinear extension exactly requires $2^n$ calls to $f$, but for several interesting set functions, e.g., graph cut, set cover, and facility location, it can be computed efficiently in $\widetilde{\mathcal O}(n^2)$ time~\citep{iyer2014monotone}.

%\sj{You are jumping from sets to matrices here, without explaining what the matrices are about. Please (1) clean up notation -- no overload, use different letters and (2) make the full story and big picture clear. Otherwise this is not understandable :)}

\section{Neural Set Function Extensions}\label{sec: neural SFEs}

%Extensions onto the hypercube $[0,1]^n$, whilst already empirically useful (see Section \ref{sec: experiments}),  are not optimally designed for combination with neural networks. Each item $i$ in the ground set $[n]$ is represented by a single scalar---the $i$th coordinate of $\xb$. Since $\xb$  is the output of a network, the scalar computation creates an acute dimensionality bottleneck in the overall architecture, which may limit  model performance. The concern over bottlenecks is highlighted in the \emph{neural algorithmic reasoning} blueprint \citep{VELICKOVIC2021100273}, which calls for neural networks that  simulate classical algorithms so that all internal representations are high dimensional vectors. 

This section builds on the scalar SFE framework---where each item $i$ in the ground set $[n]$ is represented by a single scalar---to develop extensions that use high-dimensional embeddings to avoid introducing low-dimensional bottlenecks into neural network architectures. The core motivation that lifting problems into higher dimensions can make them easier is not unique to deep learning. For instance, it also underlies kernel methods \citep{shawe2004kernel} and the \emph{lift-and-project} method for integer programming \citep{lovasz1991cones}.

Our method takes inspiration from prior successes of semi-definite programming for combinatorial optimization \citep{goemans1995improved} by extending onto $\calX = \mathbb{S}^{n}_+$, the set of $n \times n$ positive semi-definite (PSD) matrices. With this domain, each item is represented by a vector, not a scalar.  

\subsection{Lifting Set Function Extensions to Higher Dimensions}
% \AL{We should make the rank of $X$ more explicit in the discussion below. That is, rather than saying that we project in $\mathbb{S}^{n}_+$ we should talk about the set of rank-$k$ positive definite matrices $\mathbb{S}^{n,k}_+$. I think this is important because: 1) it provides a cleaner generalization of the previous section (question: do we obtain scalar SFE when $k=1$?) and 2) the constant gives us a way to trade-off computational complexity for complexity of representation. We should argue that we choose $k=O(1)$.}
% \NK{We don't get the scalar SFE for k = 1 so we can't quite make that argument without getting into the details, idk if we have enough space for it?}

We embed sets into $\mathbb{S}^{n}_+$ via the map $e(S)=\mathbf{1}_S \mathbf{1}_S^\top$.
%We may, for example, embed sets into $\calZ$ using the map $e(S) = \text{diag}(\mathbf{1}_S)$ where the $\text{diag}$ operator maps a vector to its corresponding diagonal matrix. 
To define extensions on this matrix domain, we translate the linear programming approach of Section \ref{sec: scalar SFEs} into an analogous SDP formulation: 
%Since the matrix analog of the linear program  is the semi-definite program (SDP), we aim to identify an SDP counterpart of the primal LP in Section \ref{sec: scalar SFEs}. Consider the following SDP. \sj{this is informal}
% The SDP we consider is the following:
%
 \begin{align}
     \max_{\mathbf{Z} 	\succeq 0, b \in \mathbb{R}} \{\text{Tr}(\mathbf{X^\top Z}) + b\} \text{ subject to } \frac{1}{2}\text{Tr}((\mathbf{1}_S \mathbf{1}_T^\top + \mathbf{1}_T \mathbf{1}_S^\top) \mathbf{Z}) + b \leq f(S\cap T) \text{ for } S,T \subseteq [n], \tag{primal SDP}
 \end{align}
where we switch from lower case letters to upper case since we are now using matrices. 
%To understand why this choice of primal SDP is a suitable analog of the primal LP, consider \JR{..... Todo!}  
%The set of matrices $\mathbf{Z}$ that are feasible solutions of the primal SDP is called the \emph{spectrahedron} of $f$. 
Next, we show that this choice of primal SDP is a natural analog of the original LP that provides the right correspondences between vectors and matrices by proving that primal LP feasible solutions correspond to primal SDP feasible solutions with the same objective value (see Appendix \ref{app: primaldual} for a discussion on the SDP and its dual). To state the result, note that the embedding $e(S)=\mathbf{1}_S \mathbf{1}_S^\top$ is a particular case of the correspondence $\xb \in [0,1]^n \mapsto \sqrt{\xb} \sqrt{\xb}^\top$.
%in a natural way. First, it immediately seen that the SDP objective $ \text{Tr}(\mathbf{ZX})$ recovers the LP objective value $\zb^\top \xb$ by passing to the diagonal matrix embedding $\mathbf{Z}=\text{diag}(\zb)$ and $\mathbf{X}=\text{diag}(\xb)$.  More precisely, we will show that the feasible solutions of the primal LP are a subset of the solutions to the primal SDP.

\begin{proposition}\label{prop:sdp_lp}
(Containment of LP in SDP) For any $\xb \in [0,1]^n$, 
%let $\mathcal{P}_f$ be the polyhedron of primal LP feasible solutions, and $\mathcal{S}_f$ be the spectrahedron of primal SDP feasible solutions for 
define $\mathbf{X} = \sqrt{\xb} \sqrt{\xb}^\top$ with the square-root taken entry-wise.  Then, for any $(\zb,b) \in \mathbb{R}^n_+ \times \mathbb{R}$ that is primal LP feasible, the pair  $(\mathbf{Z} ,b)$ where $\mathbf{Z}=\text{diag}(\zb)$, is primal SDP feasible and the objective values agree:  $\text{Tr}(\mathbf{X^\top Z}) =\zb^\top \xb$.
\end{proposition}
%\NK{this is not a textbook SDP; there is no reference for those claims, it's our construction}.
%
Proposition~\ref{prop:sdp_lp} establishes that the primal SDP feasible set is a \textit{spectrahedral lift} of the positive primal LP feasible set, i.e., feasible solutions of the primal LP lead to feasible solutions of the primal SDP.
%, and that the LP and SDP objectives have the same value. In other words, the primal SDP is a natural generalization of the primal LP, and the captures the same feasible solutions as the LP.
%\JR{we need to explain why not to be worried about the non-negativity requirement.} That both feasibility, and objective value are preserved when passed through the $\text{diag}$ embedding function justifies the choice of primal SDP as a natural extension of the original LP. \sj{though the optimal solutions may be different...} 
As with scalar SFEs, to define neural SFEs we consider the dual SDP:
\begin{align}
    &  \min_{ \{y_{S,T} \geq 0 \}}\sum_{ S,T
     \subseteq [n]}  y_{S, T}  f(S\cap T) \ \text{ subject to } \  \mathbf{X}\preceq \sum_{S,T \subseteq [n]} \frac{1}{2}y_{S , T}(\mathbf{1}_{S} \mathbf{1}_{T}^\top+ \mathbf{1}_{T} \mathbf{1}_{S}^\top) \ \ \text{ and } \sum_{ S, T \subseteq [n]} y_{{S,T}} = 1 \tag{dual SDP}
\end{align}
We demonstrate that for suitable $\mathbf{X}$ this SDP has feasible solutions via an explicit construction in Section \ref{sec: constructing neural SFEs}. This leads us to define a neural SFE which, as with scalar SFEs, is given by a feasible solution to the dual SDP that satisfies the extension property whose coefficients are continuous in $\mathbf{X}$:
\begin{definition}[Neural SFE]
A neural set function extension of $f$  at a point $\mathbf{X} \in \mathbb{S}^{n}_+$ is defined as
    $$ \barf(\mathbf{X}) \triangleq \sum_{S,T \subseteq [n]} p_\mathbf{X}(S , T) f(S\cap T), $$
where $y_{S, T}=p_\mathbf{X}(S , T)$ is a feasible solution to the dual SDP and for all $S,T \subseteq [n]$: 1) $p_\mathbf{X}(S, T)$ is continuous at $\mathbf{X}$ and 2) it is valid, i.e., $\barf(\mathbf{1}_S\mathbf{1}_S^{\top})=f(S)$ for all $S \subseteq [n]$.
    % \begin{enumerate}
    %     \item (Differentiable) $p_\mathbf{X}(S, T)$ is a continuous and a.e. differentiable function of $\mathbf{X}$,
    %     \item (Extension)  $\barf(\mathbf{1}_S\mathbf{1}_S^{\top})=f(S)$.
    % \end{enumerate}  
\end{definition}

%\JR{Discuss the two different vector to matrix embeddings here---one applied to $\mathbf{Z}$ and one to $\mathbf{X}$.}
%\sj{Naming: scalar SFE and Neural SFE is a bit strange as a combination. Maybe vector and matrix?}

%Suppose we are given a matrix $\mathbf{C} \in \mathbb{S}^{n}_+ $, where $ \mathbb{S}^{n}_+$ is the set of symmetric PSD matrices. We define the following semidefinite program:

\subsection{Constructing Neural Set Function Extensions}\label{sec: constructing neural SFEs}

We constructed a number of explicit examples of scalar SFEs   in Section \ref{sec: constructing scalar SFEs}. For neural SFEs we employ a different strategy. Instead of providing individual examples of neural SFEs, we develop a single recipe for converting \emph{any} scalar SFE into a corresponding neural SFE. Doing so allows us to build on the variety of scalar SFEs and provides an additional connection between scalar and neural SFEs. In Section \ref{sec: experiments} we show the empirical superiority of neural SFEs over their scalar counterparts. 

Our construction is given in the following proposition:
%PREVIOUS VERSION ______________________________
% \begin{proposition}
% Let $p_\xb$ induce a scalar SFE of $f$. For $\mathbf{X} \in \mathbb{S}_+^n$ with distinct eigenvalues, consider the eigendecomposition $\mathbf{X} = \sum_{i=1}^n \lambda_i \xb_i \xb_i^\top$ and fix
% %
% \begin{align}
%   p_\mathbf{X}(S, T) = \sum_{i=1}^n \lambda_i \, p_{\xb_i}(S) p_{\xb_i}(T)   \text{ for all } S,T \subseteq [n]. 
% \end{align}
% %
% Then, $ p_\mathbf{X}$ defines a neural SFE $\barf$ at $\mathbf{X}$.
% \end{proposition}
% %
% See Appendix \ref{app: neural SFE proofs}  for proof. The continuity of $\barf$ follows from a variant of the Davis--Kahan theorem \citep{yu2015useful}, which uses the distinct eigenvalue assumption. For efficiency, in practice we do not use all $n$ eigenvectors, and use only the $k$ with largest eigenvalue. This is justified by Figure \ref{fig: runtime ablation}, which shows that in practical applications $\mathbf{X}$ often has a rapidly decaying spectrum. %\sj{since $X$ is the input you cannot claim it is always effectively low rank}
% % there are only a few large eigenvalues. 
% PREVIOUS VERSION ____________________________

\begin{proposition}
Let $p_\xb$ induce a scalar SFE of $f$. For $\mathbf{X} \in \mathbb{S}_+^n$, consider a decomposition $\mathbf{X} = \sum_{i=1}^n \lambda_i \xb_i \xb_i^\top$ and fix
\begin{align}
  p_\mathbf{X}(S, T) = \sum_{i=1}^n \lambda_i \, p_{\xb_i}(S) p_{\xb_i}(T)   \text{ for all } S,T \subseteq [n]. 
\end{align}
Then, $ p_\mathbf{X}$ defines a neural SFE $\barf$ at $\mathbf{X}$.
\end{proposition}
See Appendix \ref{app: neural SFE proofs}  for proof. The choice of decomposition will give rise to different extensions. Here, we instantiate our neural extensions using the eigendecomposition of $\mathbf{X}$. Since eigenvectors may not belong to $[0,1]^n$ we reparameterize by first applying a sigmoid function before computing the scalar extension distribution $p_\xb$. In practice we found that neural SFEs work just as well even without this sigmoid function---i.e., allowing scalar SFEs to be evaluated outside of $[0,1]^n$. The continuity of the neural SFE $\barf$ when using the eigendecomposition follows from a variant of the Davis--Kahan theorem \citep{yu2015useful}, which requires the additional assumption that the eigenvalues of $\mathbf{x}$ are distinct. For efficiency, in practice we do not use all $n$ eigenvectors, and use only the $k$ with largest eigenvalue.
This is justified by Figure \ref{fig: runtime ablation}, which shows that in practical applications $\mathbf{X}$ often has a rapidly decaying spectrum. %\sj{since $X$ is the input you cannot claim it is always effectively low rank}
% there are only a few large eigenvalues. 

%In general, the eigenvectors are not differentiable with respect to $\mathbf{X}$, therefore we apply a stop-grad operator to $\xb_i$---i.e., treat it as a constant---and let gradients flow only through the eigenvalues $\lambda_i$. \AL{This is not clear to me. Can you explain it? Also perhaps we can move the argument to the appendix?} For computational efficiency purposes, we use the power method to approximate eigenvectors \citep{golub2013matrix} via a polynomial function in $\mathbf{X}$. Since automatic differentiation packages can easily handle polynomials we do not use the stop-gradient operator in practice.

Evaluating a neural SFE requires an accessible closed-form expression, the precise form of which depends on the underlying scalar SFE. Further, from the definition of Neural SFEs we see that if a scalar SFE is supported on sets with a property that is closed under intersection (e.g., bounded cardinality), then the supporting sets of the corresponding neural SFE will also inherit that property. This implies that the neural counterparts of the Lov\'asz, bounded cardinality Lov\'asz, and singleton/permutation extensions have the same support  as their scalar  counterparts.
An immediate corollary is that we can easily compute the neural counterpart of the Lov\'asz extension which has a simple closed form:

\begin{corollary}
 For $\mathbf{X} \in \mathbb{S}_+^n$ consider the eigendecomposition $\mathbf{X} = \sum_{i=1}^n \lambda_i \xb_i \xb_i^\top$. Let $p_{\xb_i}$ be as in the Lov\'asz extension: $p_{\xb_i}(S_{ij}) = \sigma(x_{i,j}) - \sigma(x_{i,j+1})$, where  $\sigma$ is the sigmoid function, and $\xb_i$ is sorted so $x_{i,1} \geq \ldots \geq x_{i,n}$ and $S_{ij} = \{1, \ldots , j\}$, with $p_{\xb_i}(S) = 0$ for all other sets. Then, the neural Lov\'asz extension is:
\begin{align}
  \barf (\mathbf{X}) = \sum_{i,j=1}^n \lambda_i p_{\xb_i}(S_{ij})  \cdot \bigg ( p_{\xb_i}(S_{ij}) +2 \sum_{\ell: \ell>j} p_{\xb_i}(S_{i\ell}) \bigg ) \cdot f(S_{ij}).
\end{align}
\end{corollary}
\paragraph{Complexity and obtaining sets as solutions.}
In general, the neural SFE relies on all pairwise intersections $S \cap T$ of the scalar SFE sets, requiring $O(m^2)$ evaluations of $f$  when the scalar SFE is supported on $m$ sets. However, when the scalar SFE is supported on a family of sets that is closed under intersection---e.g., the Lov\'asz and singleton extensions---the corresponding neural SFE requires only $O(m)$ function evaluations. 
%Examples include the neural Lov\'asz extension defined above since the scalar SFE is supported on a chain of level sets, as well as extensions on singletons since their intersections are empty. 
Discrete solutions can be obtained efficiently by returning the best set out of all scalar SFEs $p_{\xb_i}$.% \AL{What is k, explain again?}

\section{Experiments}\label{sec: experiments}
We experiment with SFEs as loss functions in neural network pipelines on discrete objectives arising in combinatorial and vision tasks. For combinatorial optimization, SFEs network training with a continuous version of the objective without supervision. For supervised image classification, they allow us to directly relax the training error instead of optimizing a proxy like cross entropy.

% \begin{table}[]
% \resizebox{0.99\textwidth}{!}{
% \begin{tabular}{lcll|cll}
%                  & \multicolumn{3}{c}{\textbf{Max-Cut}}                                             & \multicolumn{3}{c|}{\textbf{Max-Clique}}                                         
%                  & Enzymes                            & Proteins        & IMDB-Binary      & Enzymes                            & Proteins        & IMDB-Binary      \\ \hline
% Straight-through \citep{bengio2013estimating} & $0.877_{\pm 0.084} $                     & $0.903_{\pm 0.174}$ & $0.531_{\pm 0.253}$& $0.556_{\pm0.128}  $                  & $0.309_{\pm0.443}$  & $0.795_{\pm 0.354}$ \\
% Erd\H{o}s \citep{karalias2020erdos}          & $0.946_{\pm 0.032}  $                   & $0.930_{\pm 0.063}$ & $0.992_{\pm 0.019}$   & $0.837_{\pm0.172} $                  & $0.900_{\pm0.139}$  & $0.893_{\pm0.149}$    \\
% Greedy           & $0.969_{\pm 0.025} $                   & $0.977_{\pm 0.034}$ & $1.000_{\pm 0.000}$   & $0.973_{\pm 0.076}$                  & $0.978_{\pm 0.072}$ & $0.950_{\pm 0.071}$ \\ \hline
% Lovasz           & \multicolumn{1}{l}{$0.968_{\pm 0.200}$} & $0.971_{\pm 0.021}$ & $0.958_{\pm0.048}$   & \multicolumn{1}{l}{$0.720_{\pm 0.247}$} &  $0.875_{\pm 0.125}$  & $0.903_{\pm 0.214}$
% \end{tabular}
% }\caption{\textbf{Unsupervised combinatorial optimization:} \JR{TODO}}
% \end{table}
%%%%%%%%%%%%%%%%%%%%%%%%%%%%%%%%%%%%%%%%%%%%%
\begin{table}
\centering
\resizebox{0.99\textwidth}{!}{
\begin{tabular}{lccc cc} \bottomrule 

\toprule
                 & \multicolumn{5}{c}{\hspace{-12mm}\textbf{Maximum Clique}}                                                                                                                                         \\ 
                 & ENZYMES                                                    & PROTEINS                & IMDB-Binary          & MUTAG                                 & COLLAB                           \\  \midrule
%Greedy heuristic      & 0                                       & 0    & 0 & 0    & 0     \\
\midrule

Straight-through \citep{bengio2013estimating} &   $0.725_{\pm0.268}$                                        & $0.722_{\pm 0.26}
$    &  $0.917_{\pm 0.253}$ & $0.965_{\pm 0.162}$    &   $0.856_{\pm 0.221}$   \\
Erd\H{o}s \citep{karalias2020erdos}        & $0.883_{\pm 0.156}$                                      & $0.905_{\pm 0.133}$   & $0.936_{\pm 0.175}$ & $1.000_{\pm 0.000}$    & $0.852
_{\pm 0.212}$     \\
REINFORCE  \citep{williams1992simple}      & $0.751_{\pm 0.301}$                                      & $0.725_{\pm 0.285}$    & $0.881_{\pm 0.240}$ & $1.000_{\pm 0.000}$    & $0.781_{\pm 0.316}$     \\

 \midrule
 Lov\'{a}sz scalar SFE            & \multicolumn{1}{l}{$0.723_{\pm 0.272}$} & $0.778_{\pm 0.270}$ & $0.975_{\pm 0.125}$  & \multicolumn{1}{l}{$0.977_{\pm0.125}$ } & $0.855_{\pm0.225}$ \\
Lov\'{a}sz neural SFE            & \multicolumn{1}{l}{$0.933_{\pm 0.148}$} & $0.926_{\pm0.165}$ & $0.961_{\pm 0.143}$   & \multicolumn{1}{l}{$1.000_{\pm 0.000}$} & $0.864_{\pm 0.205}$  \\
\bottomrule

\iffalse
 \toprule
                 & \multicolumn{5}{c}{\hspace{-12mm}\textbf{Maximum Cut}}                                                                                                                                    \\ 
                 & ENZYMES                                                    & PROTEINS                & IMDB-Binary          & MUTAG                                 & COLLAB                      \\  \midrule
%Greedy heuristic      & 0                                       & 0    & 0 & 0    & 0      \\
\midrule
Straight-through \citep{bengio2013estimating} &   $0.972_{\pm 0.065} $                                   & $0.975_{\pm 0.052}$   & $0.956_{\pm 0.183}$  & $0.960_{\pm 0.049}$    & $0.960_{\pm 0.042}$      \\
Erd\H{o}s \citep{karalias2020erdos}        & $0.964_{\pm 0.051} $                                       & $0.971_{\pm 0.067} $    & $0.992_{\pm 0.024}$ & $0.916_{\pm 0.074}$    & $0.999_{\pm 0.032}$       \\
REINFORCE \citep{williams1992simple}       & $0.982_{\pm 0.043}$                                       & $0.986_{\pm 0.054}$   & $0.999_{\pm 0.022}$ & $0.973_{\pm 0.046}$    & $0.956_{\pm 0.046}$    \\

 \midrule
Lov\'{a}sz scalar SFE            & \multicolumn{1}{l}{$0.974_{\pm 0.048}$} & $0.9777_{\pm 0.032}$ & $0.959_{\pm 0.057}$   & $0.972_{\pm 0.042}$ & $0.998_{\pm 0.053}$  \\
Lov\'{a}sz neural SFE            & \multicolumn{1}{l}{$0.975_{\pm 0.042}$} & $0.963_{\pm 0.049}$ & $0.936_{\pm 0.079}$   & $0.975_{\pm 0.053}$ & $0.966_{\pm 0.060}$  \\
\bottomrule
 \fi
 
 \toprule
                 & \multicolumn{5}{c}{\hspace{-12mm}\textbf{Maximum Independent Set}}                                                                                          \\ 
                 & ENZYMES                                                    & PROTEINS                & IMDB-Binary          & MUTAG                                 & COLLAB            \\ 
                  \midrule
%Greedy heuristic      & 0                                       & 0    & 0 & 0    & 0      \\
\midrule
Straight-through \citep{bengio2013estimating}&   $0.505_{\pm 0.244}$                                     & $0.430_{\pm 0.252}$    & $0.701_{\pm 0.252}$ & $0.721_{\pm 0.257}$    & $0.331_{\pm 0.260}$      \\
Erd\H{o}s \citep{karalias2020erdos}        & $0.821_{\pm 0.124}$                                       & $0.903_{\pm 0.114}$    & $0.515
_{\pm 0.310}$ & $0.939_{\pm 0.069}$    & $0.886_{\pm 0.198}$       \\
REINFORCE \citep{williams1992simple}    & $0.617_{\pm 0.214}$                                      & $0.579_{\pm 0.340}$    & $0.899_{\pm 0.275}$ & $0.744_{\pm 0.121}$ &   $0.053_{\pm 0.164}$     \\
 \midrule
Lov\'{a}sz scalar SFE            & $0.311_{\pm 0.289}$ & $0.462_{\pm 0.260}$ & $0.716_{\pm 0.269}$   & $0.737_{\pm 0.154}$ & $0.302_{\pm 0.238}$  \\
Lov\'{a}sz neural SFE            & $0.775_{\pm 0.155}$ & $0.729_{\pm 0.205}$ & $0.679_{\pm 0.287}$   & $0.854_{\pm 0.132}$ & $0.392_{\pm 0.253}$ \\
\bottomrule
\end{tabular}} 
\begin{center}
\caption{\textbf{Unsupervised neural combinatorial optimization}: Approximation ratios for combinatorial problems. Values closer to 1 are better ($\uparrow$). Neural SFEs are competitive with other methods, and consistently improve over vector SFEs. 
%\sj{greedy is too good for these problems. Maybe better sth like balanced cut? what would happen if you use the Goemans-Williamson SDP for maxcut?}
} \label{table:combinatorial}
\end{center}
\vspace{-20pt}
\end{table}
%
% \begin{tabular}{*5l}    \toprule
% \emph{name} & \emph{foo} &&&  \\\midrule
% Models    & A  & B  & C  & D  \\ 
% \rowcolor{blue!50} Model $X$ & X1 & X2 & X3 & X4\\ 
% \rowcolor{green!50} Model $Y$ & Y1 & Y2 & Y3 & Y4\\\bottomrule
%  \hline
% \end{tabular}

\subsection{Unsupervised Neural Combinatorial Optimization}

We begin by evaluating the suitability of neural SFEs for unsupervised learning of neural solvers for combinatorial optimization problems on graphs. We use the ENZYMES, PROTEINS, IMDB, MUTAG, and COLLAB datasets from the TUDatasets benchmark \citep{morris2020tudataset}, using a 60/30/10 split for train/test/val. We test on two problems: finding maximum cliques, and maximum independent sets. 
%Our models are trained in a \emph{minimal setting} in order to emphasize the contribution of the loss function and are not meant to establish state of the art results.
We compare with three neural network based methods. We compare to two common approaches for backpropogating through discrete functions: the REINFORCE algorithm \citep{williams1992simple}, and the Straight-Through estimator \citep{bengio2013estimating}. %We emphasize that it is uncommon to learn neural solvers in an \emph{unsupervised} manner as we do; indeed a recent survey \citep{bengio2021machine} noted that \emph{its immediate use seems difficult}. 
The third is the recently proposed probabilistic penalty relaxation \citep{karalias2020erdos} for combinatorial optimization objectives. All methods use the same GNN backbone, comprising a single GAT layer \citep{velivckovic2018graph} followed by  multiple gated graph convolution layers \cite{li2015gated}. %Finally, we also record the performance of a (non-neural) greedy heuristic. %\sj{say which GNN you use and cite}

In all cases, given an input graph $G=(V,E)$ with $|V|=n$ nodes,  a GNN produces an embedding for each node: $\mathbf{X} \in \mathbb{R}^{n \times d}$. For scalar SFEs $d=1$, while for neural SFEs we consider $\mathbf{X}\mathbf{X}^\top$ in order to produce an $n \times n$ PSD matrix, which is passed as input to the SFE  $\barf$. The  set function $f$  used is problem dependent, which we discuss below.  Finally, see Appendix \ref{app: unsup comb} for training and hyper-parameter optimization details, and Appendix \ref{app: general expt setup} for details on data, hardware, and software.

   % For the maximum cut problem, we use the Lova\'{s}z extension to train the network to minimize the negative cut function $f(S;G)=-\text{Cut}(S;G)$, where $S$ are subsets of the graph $G$.
%For the Erd\H{o}s framework, we derive a continuous relaxation by calculating the expected value of the discrete objective. For the Straight-Through estimator, we sample from the distribution of level sets of $\mathbf{x}$ using random thresholds and average the evaluations of the objective on those. Then we backpropagate through the sampling operation by treating it as the identity mapping.  

%\textbf{Maximum Cut.}
%A set $S \subseteq V$ partitions a graph into two pieces---$S$ and $T=G \setminus S$. The cut size is the number of edges $(v,v')$ with $v \in S$ and $v' \in T$. The MaxCut problem is to find $S$ that creates the largest cut. That is, maximize $f(S)=\text{Cut}(S)$.

\textbf{Maximum Clique.}
A set $S \subseteq V$ is a clique of $G=(V,E)$ if $(i,j) \in E$ for all $i,j \in S$. The MaxClique problem is to find the largest set $S$ that is a clique: i.e., $f(S)=|S|\cdot \mathbf{1}\{ S \text{ a clique} \}$.

\textbf{Maximum Independent Set (MIS).}
A set $S \subseteq V$ is an independent set of $G=(V,E)$ if $(i,j) \notin E$ for all $i,j \in S$. The goal is to find the largest $S$ in the graph that is independent, i.e., $f(S)=|S|\cdot \mathbf{1}\{ S \text{ an ind. set} \}$.
%Equivalent to MaxClique on the complement graph, 
MIS differs significantly from MaxClique due to its high heterophily. 

\textbf{Results.} 
Table \ref{table:combinatorial} displays the mean and standard deviation of the approximation ratio $f(S)/f(S^*)$ of the solver solution $S$ and an optimal $S^*$ on the test set graphs. The neural Lova\'{s}z extension outperforms its scalar counterpart in 8 out of 10 cases, often by significant margins, for instance improving a score of $0.778$ on PROTEINS MaxClique to $0.926$. The neural SFE proved effective at boosting poor scalar SFE performance, e.g., $0.311$ on ENZYMES MIS, to the competitive performance of $0.775$. Neural Lova\'{s}z outperformed or equalled  and straight-through in 9 out of 10 cases, and the method of \cite{karalias2020erdos} in 6 out of 10.

%For the maximum clique problem, the Lova\'{s}z extension is trained directly with a function that incorporates the constraint as a multiplicative term. For the Erd\H{o}s framework, we use the loss that the authors provided in the paper. For the Straight-Through, we use the same discrete function as in the Lova\'{s}z extension.

% We train with a 60-20-20 split and report the test performance of the models that achieved the best validation accuracy.
%We report the mean approximation ratio across the test set along with the standard deviation. We train modern graph neural network architectures on real world datasets with graph sizes of up to a few hundred nodes. 
%For details of the setup, we refer the reader to Appendix \ref{app: unsup comb}. 

%\paragraph{Results.}
%The results reported on Table \ref{table:combinatorial}  suggest that the Lova\'{s}z extension indeed provides a way to directly minimize a set function. It is able to outperform the Erd\H{o}s probabilistic loss~\citep{karalias2020erdos} on two datasets on max-cut and is competitive on  max-clique. On the other hand, the Straight-Through estimator delivers inconsistent results, especially on max-clique where an objective with constraints is optimized. 

% \subsection{Generalization}
% Table (or figure) of train and test performance for a couple of problems.
% Alternatively, this could be an E-R dataset with varying size and solving max cut as it grows larger

\subsection{Constraint Satisfaction Problems}\label{sec: constraint}

 \begin{figure*}[t] %{6.5cm}
  \begin{center}
    \includegraphics[width=\textwidth]{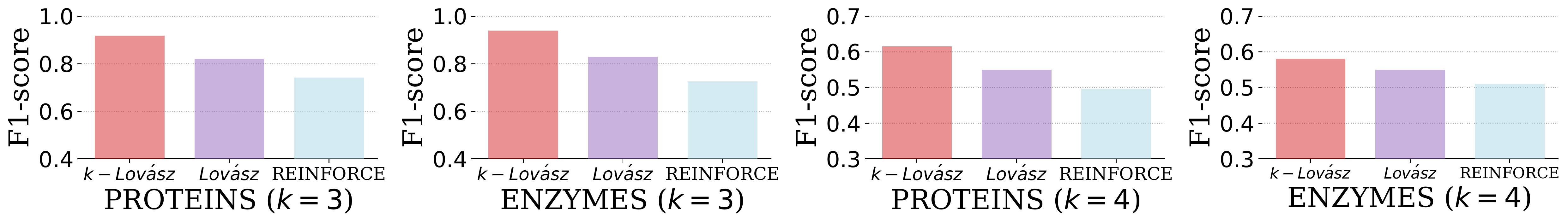}
      \vspace{-10pt}
    \caption{\textbf{$k$-clique constraint satisfaction:} higher F1-score is better. The $k$-bounded cardinality Lovasz extension is better aligned with the task  and significantly improves over the Lov\'{a}sz extension.}
    \label{fig: k clique}
  \end{center}
  \vspace{-10pt}
\end{figure*}

 \begin{figure*}[t] %{6.5cm}
  \begin{center}
      \includegraphics[width=0.245\textwidth]{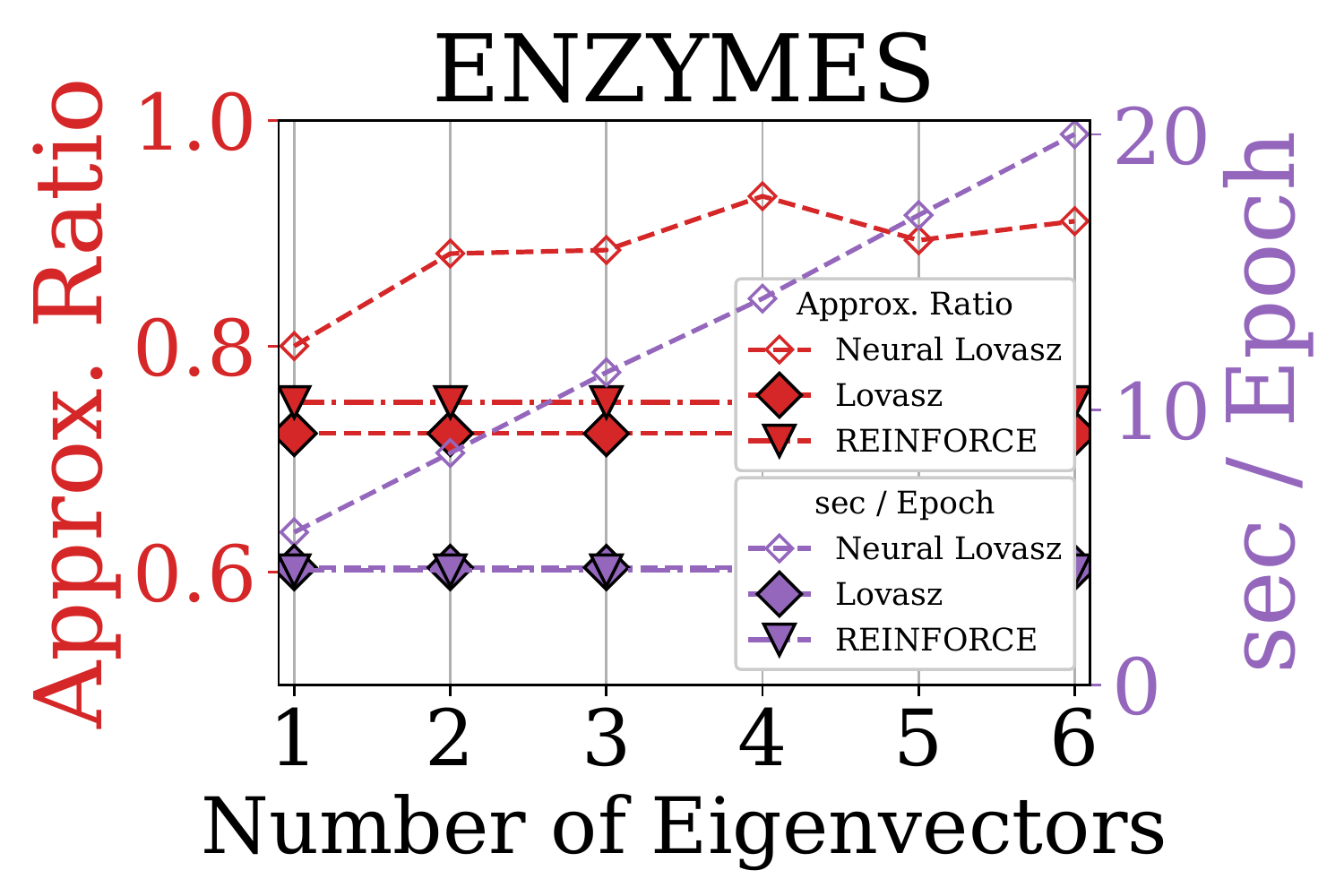}    \includegraphics[width=0.245\textwidth]{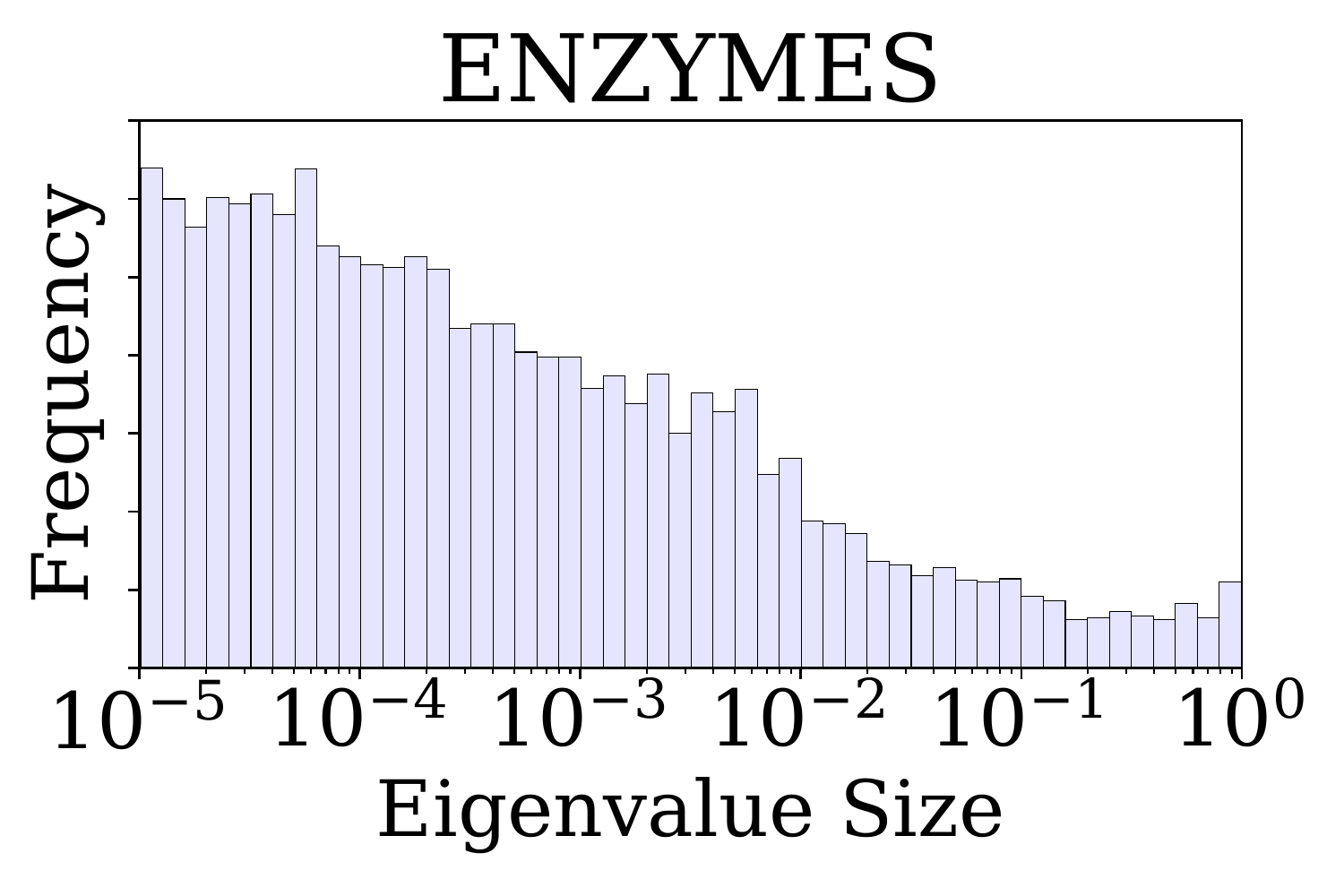}   
    \includegraphics[width=0.245\textwidth]{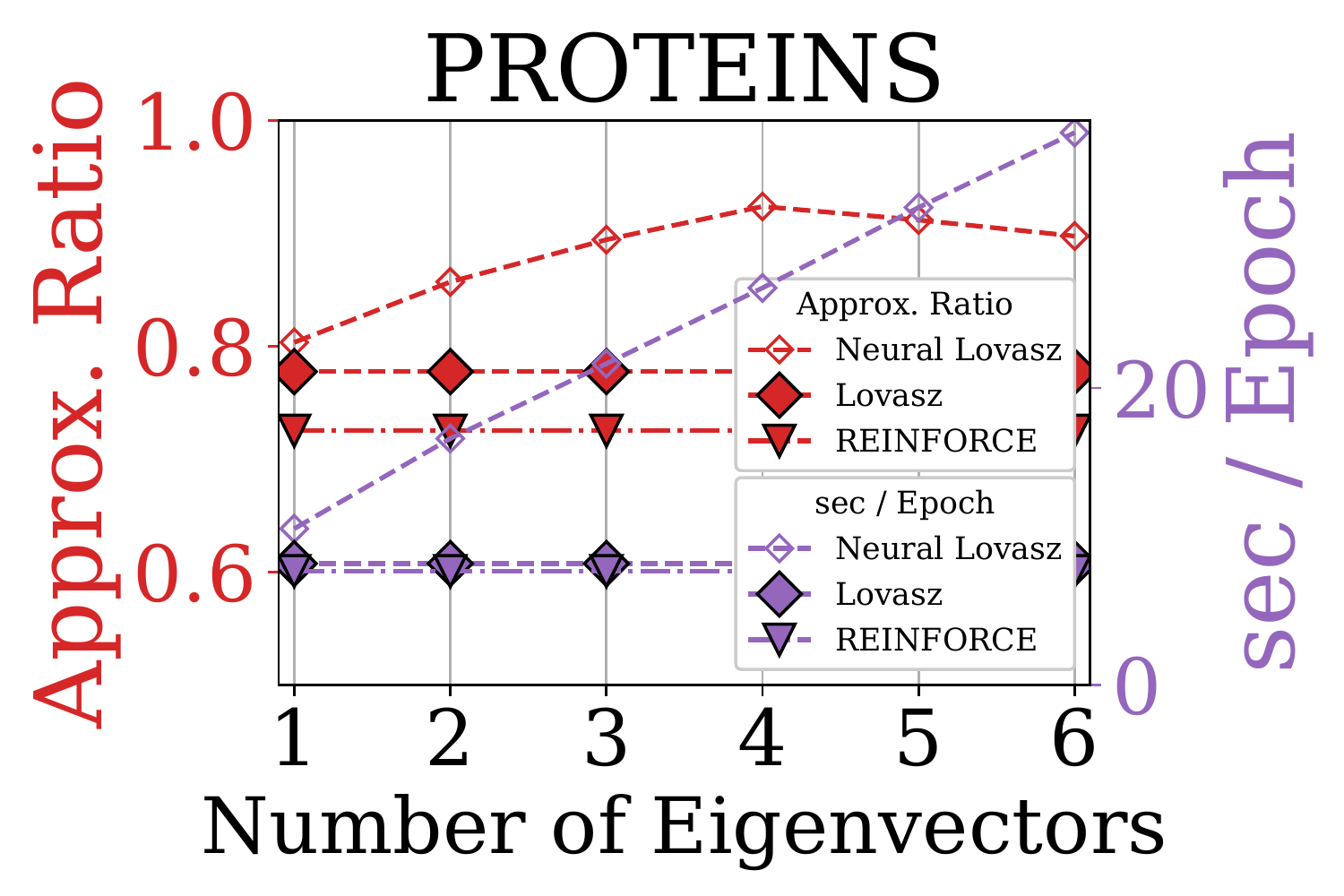}
    \includegraphics[width=0.245\textwidth]{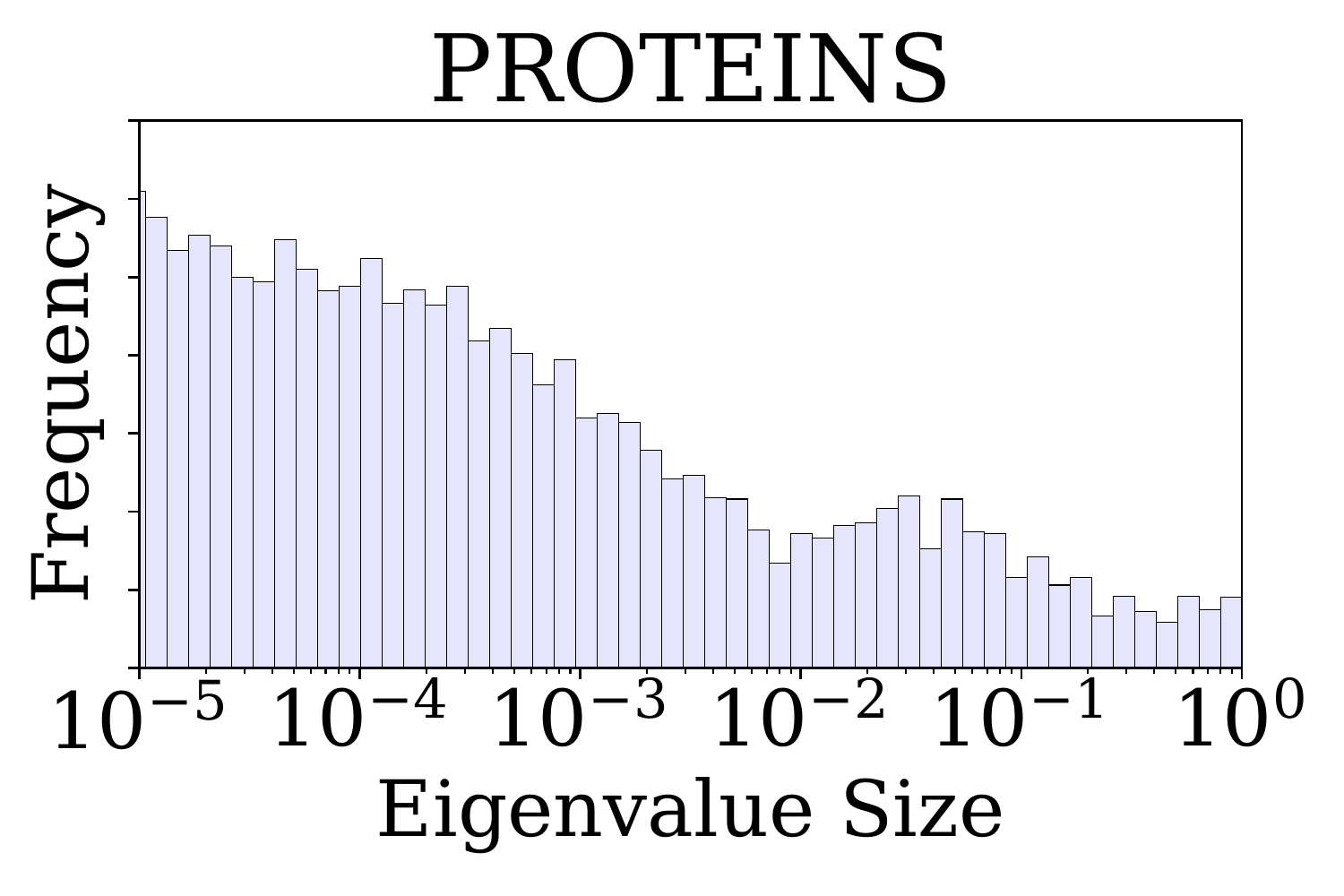}
      \vspace{-10pt}
    \caption{\textbf{Left:} Runtime and performance of neural SFEs on MaxClique using different numbers of eigenvectors. \textbf{Right:} Histogram of spectrum of matrix $\mathbf{X}$, outputted by a GNN trained on MaxClique.} %Performance increases with more eigenvectors, and eventually saturates, while runtime continues to grow.}
    \label{fig: runtime ablation}
  \end{center}
  \vspace{-10pt}
\end{figure*}

Constraint satisfaction problems ask if there exists a set satisfying a given set of conditions \citep{kumar1992algorithms,cappart}.
% Constraint satisfaction problems are ubiquitous in both theoretical and applied fields of computer science and include important problems such as boolean satisfiability. 
In this section, we apply SFEs to the $k$-clique problem: given a graph, determine if it contains a clique of size $k$ or more. We test on the  ENZYMES and PROTEINS  datasets. Since satisfiability is  a binary classification problem we evaluate  using F1 score. 
%\sj{which dataset did you use?}

%We demonstrate the adaptability of our framework by building the cardinality into  $\barf$. Specifically, we take $\Omega$ to be the family of subsets of the node set $V$ of size exactly $k$, and use the fixed set size extension discussed in Section \ref{sec: examples}. The goal is to learn to  output \texttt{true} if we successfully find a $k$-clique, and \texttt{false} otherwise. We compare this approach to the Lov\'{a}sz extension, which searches over sets of different cardinalities,  and two non-neural baselines: 1) Greedily constructing a set $S$, starting with $S =\emptyset$, iterating over all nodes $i$ and updating $S \leftarrow  S \cup \{i\}$ for  the the first $i$ for which $S \cup \{i\}$ is a clique (termination when no such $i$ exists). We return \texttt{true} if the terminal set  $S$ is a $k$-clique and \texttt{false} otherwise. 2) Randomly generating $1000$ sets  of size $k$ and returning \texttt{true} if  any of these sets is a $k$-clique, and \texttt{false} otherwise.

\textbf{Results.} Figure \ref{fig: k clique} shows that by specifically searching over sets of size $k$ using the cardinality constrained Lov\'{a}sz extension from Section \ref{sec: constructing scalar SFEs}, we  significantly improve performance compared to the Lov\'{a}sz extension, and REINFORCE. This illustrates the value of SFEs in allowing task-dependent considerations (in this case a cardinality constraint) to be built into extension design. 

\begin{wrapfigure}{r}{0.45\textwidth}
  \begin{center}
    \vspace{-30pt}
    \includegraphics[width=0.44\textwidth]{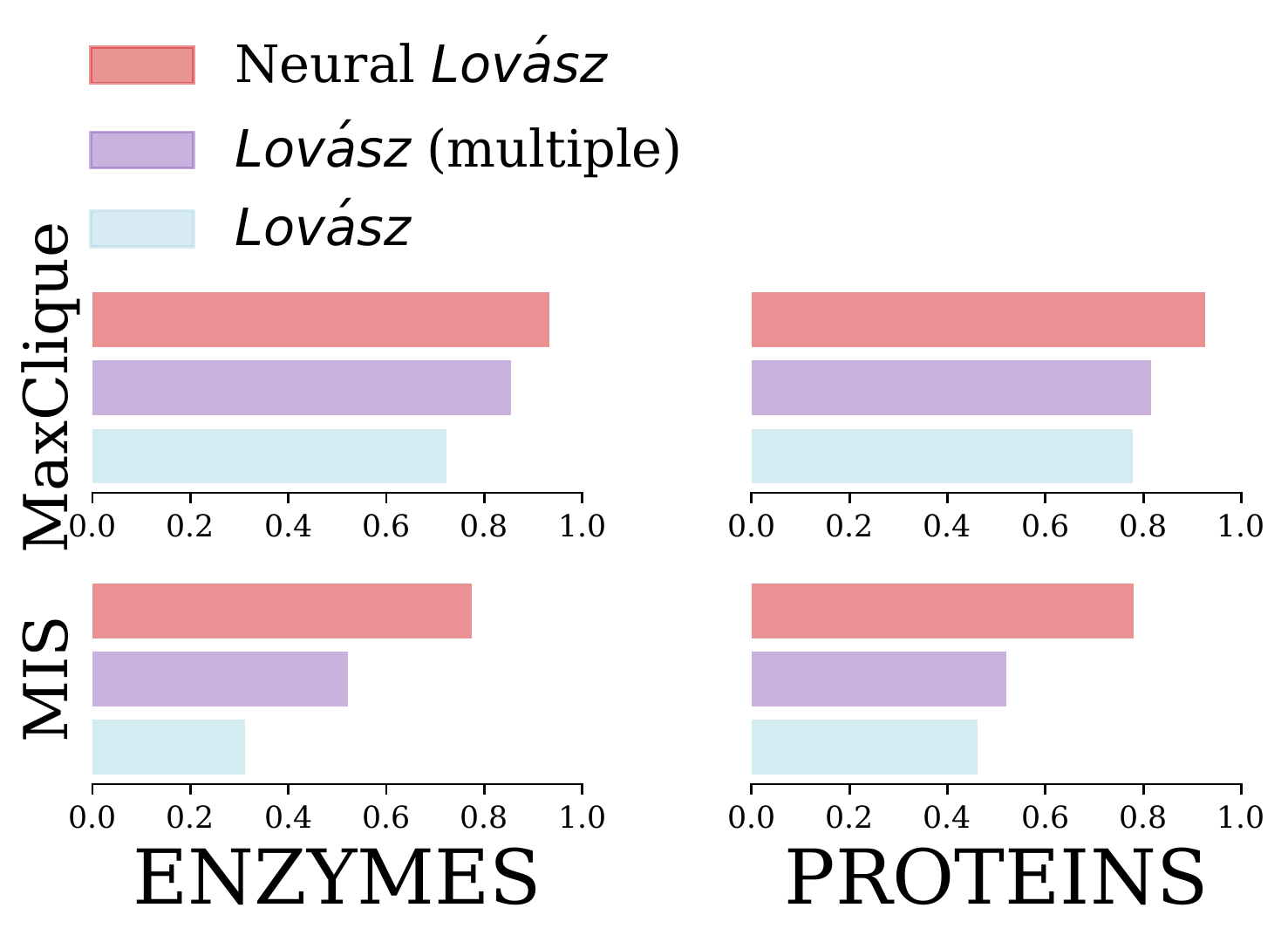}
    \caption{Neural SFEs outperform a naive alternative  high-dimensional extension.}
    \label{fig: lovasaz multi ablation}
    \vspace{-5pt}
  \end{center}
\end{wrapfigure} 
%%%%%%%%%%%%%%%%%%%%%%%%%%%%%%%%%%%%%%%%%%%%%%%%%%%%%%%%THIS THROWS A WARNING

\subsection{Training Error as a Classification Objective} 
During training the performance of a classifier $h$ is typically assessed using the training error $\frac{1}{n}\sum_{i=1}^n \mathbf{1}\{ y_i \neq h(x_i)\}$. Since training error itself is non-differentiable, it is standard to train $h$ to optimize a differentiable surrogate such as the cross-entropy loss. Here we offer an alternative training method by continuously extending the non-differentiable mapping $\hat{y} \mapsto \mathbf{1}\{ y_i \neq \hat{y} \}$.  This map is a set function defined on single item sets, so we use the singleton extension  (definition in Section \ref{sec: constructing scalar SFEs}). Our goal is 
to demonstrate that the resulting differentiable loss function closely tracks the training error, and can be used to minimize it. We do not focus on test time generalization. 
%therefore not to show an improvement over these widely used methods, or to argue that our framework is more principled. Instead, we aim to show that it is possible to derive a valid approach to supervised classification from the perspective of set function extension. To this end, we use the singleton extension example given in Section \ref{sec: examples} as  our extension of $i \mapsto \mathbf{1}\{ y \neq i \}$.
 Figure \ref{fig: error v loss} shows the results. The singleton extension loss (left plot) closely tracks the true training  error at the same numerical scale, unlike other common loss functions (see Appendix \ref{app: image classification} for setup details). While we leave further consideration to future work, training error extensions may be useful for model calibration \citep{kennedy2001bayesian} and uncertainty estimation \citep{abdar2021review}.

 \begin{figure*}[t] %{6.5cm}
  \begin{center}
    \includegraphics[width=\textwidth]{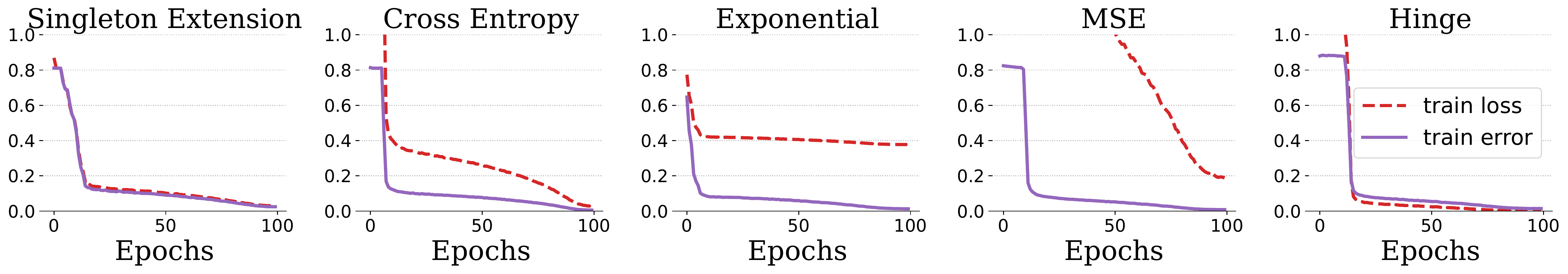}
    \includegraphics[width=\textwidth]{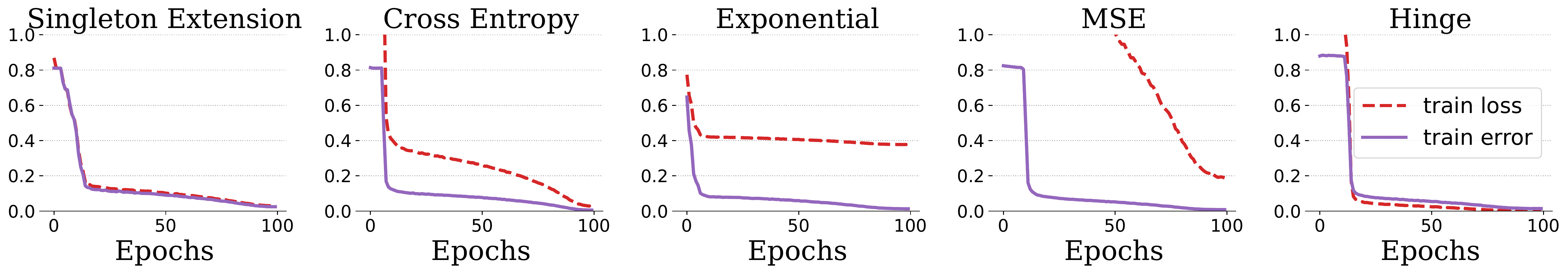}
      \vspace{-10pt}
    \caption{Top: CIFAR10. Bottom: SVHN. The singleton extension loss (left) is the only loss that approximates the true non-differentiable training error at the same numerical scale.}
    \label{fig: error v loss}
  \end{center}
  \vspace{-5pt}
\end{figure*}

%We train ResNet-18 classifiers on several vision datasets, and compare the singleton extension loss to standard supervised losses: cross-entropy, mean squared error, hinge, and exponential loss. We tune the learning rate for each loss using a simple exhaustive grid search over the interval $lr \in \{ 0.01, 0.05, 0.1, 0.2\}$ on a held out validation set. We report average test set performance over three random seeds. 

 \subsection{Ablations}
 
\textbf{Number of Eigenvectors.} Figure \ref{fig: runtime ablation} compares the runtime and performance of neural SFEs using only the top-$k$ eigenvectors from the eigendecomposition $\mathbf{X} = \sum_{i=1}^n \lambda_i \xb_i \xb_i^\top$ with $k \in \{1,2,3,4,5,6\}$ on the maximum clique problem. For both ENZYMES and PROTEINS, performance increases with $k$---easily outperforming scalar SFEs and REINFORCE---until saturation around $k=4$, while runtime grows linearly with $k$. Histograms of the eigenvalues produced by trained networks show a  rapid  decay in the spectrum,  suggesting that the smaller eigenvalues have little effect on $\barf$.

\textbf{Comparison to Naive High-Dimensional Extension.} We compare neural SFEs to a naive high-dimensional alternative which, given an $n \times d$ matrix $\mathbf{X}$ simply computes a scalar SFE on each column independently and sums them up. This naive function design is not an extension, and the dependence on the $d$ dimensions is linearly separable, in contrast to the complex non-linear interactions between columns of $\mathbf{X}$ in neural SFEs. Figure \ref{fig: lovasaz multi ablation} shows that this naive extension, whilst improving over one-dimensional extensions, performs considerably worse than neural SFEs.
%%%%%%%%%%%%%%%%%%%%%%%%%%related work
\section{Related Work} 
% \vspace{-1mm} 
\textbf{Neural combinatorial optimization}
Our experimental setup largely follows recent work on unsupervised neural combinatorial optimization \citep{karalias2020erdos, schuetz2022combinatorial, xu2020tilingnn, toenshoff2021graph, amizadeh2018learning}, where continuous relaxations of discrete objectives are utilized. In that context, it is important to take into account the key conceptual and methodological differences of our approach. For instance, in the unsupervised \emph{Erd\H{o}s goes neural} (EGN) framework from \cite{karalias2020erdos}, the probabilistic relaxation and the proposed choice of distribution can be viewed as instantiating a multilinear extension. As explained earlier, this extension is costly in the general case (since $f$ must be evaluated $2^n$ times, and summed) but can be computed efficiently in closed form in certain cases. On the other hand, our extension framework offers multiple options for efficiently computable extensions without imposing any further conditions on the set function. For example, one could efficiently (linear time in $n$) compute the scalar and neural Lov\'asz extensions of any set function with only black-box access to the function. This renders our framework more broadly applicable. Furthermore, EGN incorporates the problem constraints additively in the loss function. In contrast to that, our extension framework does not require any commitment to a specific formulation in order to obtain a differentiable loss. This provides more flexibility in modelling the problem, as we can combine the cost function and the constraints in various other ways (e.g., multiplicatively). 
%since our extensions can be used to relax discrete objectives to obtain a differentiable loss.
% In fact, the 'Erd\H{o}s' baseline included in our experiments can be viewed through the lens of SFEs as an instantiation of the multilinear extension that employs a probabilistic penalty to construct a differentiable loss from a discrete objective. 
For general background on neural combinatorial optimization, we refer the reader to the surveys  \citep{bengio2021machine, cappart2021combinatorial, mazyavkina2021reinforcement}.

\textbf{Lifting to high-dimensional spaces.} Neural SFEs are heavily inspired by the Goemans-Williamson \citep{goemans1995improved} algorithm and other SDP techniques \citep{iguchi2015tightness}, which lift problems onto higher dimensional spaces, solve them, and then project back down.
Our approach to lifting set functions to high dimensions is motivated by the algorithmic alignment principle \citep{xu2019can}: neural networks whose computations emulate classical algorithms often generalize better with improved sample complexity \citep{yan2020neural, li2020strong, xu2019can}. Emulating algorithmic and logical operations is the focus of  Neural Algorithmic Reasoning \citep{velivckovic2019neural, dudzik2022graph, deacplanners2021} and work on knowledge graphs \citep{hamilton2018embedding, ren2019query2box, arakelyan2020complex}, which also emphasize operating in higher dimensions.  

\textbf{Extensions.} Scalar SFEs use an LP formulation of the convex closure \cite[Def. 20]{el2018learning}, a classical approach for defining convex extensions of discrete functions \cite[Eq. 3.57]{murota1998discrete}. See \cite{bach2019submodular} for a study of extensions of submodular functions. The constraints of our dual LP arise in contexts from global optimization \citep{tawarmalani2002convex} to barycentric approximation and interpolation schemes in computer graphics \citep{guessab2013generalized,hormann2014barycentric}. 
Convex extensions have also been used for combinatorial penalties
with structured sparsity \citep{obozinski2012convex,obozinski2016unified}, and general minimization algorithms for set functions \citep{el2020optimal}.

\textbf{Stochastic gradient estimation.} SFEs produce gradients for $f$ requiring only black-box access. There is a wide literature on sampling-based approaches to gradient estimation, for instance the REINFORCE algorithm \citep{williams1992simple} (i.e., score function estimator). However, sampling introduces noise which can cause unstable training and convergence issues, prompting significant study of variance reducing control variates  \citep{gu2017q, liu2018action, grathwohl2018backpropagation, wu2018variance, cheng2020trajectory}. SFEs can avoid sampling (and noise) all-together, as our extensions are differentiable and can be computed deterministically. 
 %\paragraph{Differentiable Sampling.} SFEs do not rely on any analytic properties of $f$, only requiring black-box access, a trait shared with literature on differentiating through black box (non-differentiable) functions. For instance REINFORCE \citep{williams1992simple}, also known as the score function estimator, which allows for an unbiased gradient estimate via sampling. REINFORCE can suffer from unstable training and convergence issues, motivating research on variance reduction through control variates (i.e., baselines) \citep{gu2017q, liu2018action, grathwohl2018backpropagation, wu2018variance, cheng2020trajectory}. The key difference between REINFORCE and SFEs is that REINFORCE  obtains  \emph{stochastic} gradient estimates, while  SFEs give a deterministic procedure for obtaining gradients. 
% %  In principle, one could even combine REINFORCE in order to train with an approximate SFE. 
%  Lastly, REINFORCE can also typically introduce dimensionality bottlenecks due to the need for probabilities when sampling, which is not the case with neural SFEs.
A closely related, yet distinct, task is to produce gradients through sampling operations, which introduce non-differentiable nodes in neural network computation graphs. The Straight-Through Estimator \citep{bengio2013estimating}, arguably the simplest solution, treats sampling as the identity map in the backward pass, yielding biased gradient estimates. The Gumbel-Softmax trick \citep{maddison2017concrete,jang2016},  provides an alternative method to sample from categorical distributions (also benefiting from variance reduction \citep{paulus2020rao}). The trick can be seen through the lens of the more general Perturb-and-MAP framework that treats sampling as a perturbed optimization program. This framework has since been used to generalize the trick to more complex distributions \citep{paulus2020gradient} and to differentiate through the parameters of exponential families for learning and combinatorial tasks \citep{niepertimplicit}. Broadly, these techniques relax a discrete distribution into a continuous one by utilizing a noise distribution and \emph{assuming access} to a continuous loss function. SFEs are complementary to this setup, addressing the problem of designing continuous extensions.
%, where the sampling operation is viewed as a solution to a random optimization program. 

\textbf{Differentiating through convex programs and algorithms.} Recent years have seen a  surge of interest in combining neural networks with solvers (e.g., LP solvers) and/or algorithms  in differentiable end to end pipelines \citep{agrawal2019differentiable,amos2017optnet,paulus21, vlastelica2019differentiation,wang2019satnet}. Whilst sharing the algorithmic alignment motivation of SFEs, the convex programming connection is mostly cosmetic: these works directly embed solvers into network architectures, while SFEs use convex programs as an analytical tool, without requiring solver access.

%\NK{clunky infodump sentence}
%Recent work has also connected sampling from constrained exponential families in a differentiable way through implicit MLE \citep{niepertimplicit} with the literature on using solvers as neural network layers. 
%Admittedly, our use of the LP and SDP formulations may prompt the reader to identify a resemblance with such differentiable solver architectures that have seen a surge of interest in the past few years \citep{amos2017optnet,agrawal2019differentiable, vlastelica2019differentiation,wang2019satnet, paulus21}. However, we would like to emphasize one of the key difference between SFEs and the aforementioned line of research: deriving or computing an SFE \emph{does not require} employing a solver. The optimization program in our case is only used as an analytic tool that establishes the formal foundations of our extensions.

%%%%%%%%%%%%%%%%%%%%%%%%%%discussion (conclusion and acknowledgements
\vspace{-5pt}
\section{Conclusion }
We introduced Neural Set Function Extensions, a framework that enables evaluating set functions on continuous and high dimensional representations. We showed how to construct such extensions and demonstrated their viability in a range of tasks including combinatorial optimization and image classification. Notably, neural extensions deliver good results and improve over their scalar counterparts, further affirming the benefits of problem-solving in high dimensions.
\vspace{-5pt}

\section{Acknowledgements}
NK would like to thank Marwa El Halabi, Mario Sanchez, Mehmet Fatih Sahin, and Volkan Cevher for the feedback and fruitful discussions.
NK and AL would like to thank the Swiss National Science Foundation for supporting this work in the context of the project “Deep Learning for Graph-Structured Data” (grant number PZ00P2179981). SJ and JR acknowledge support from NSF CAREER award 1553284, NSF award 1717610, and the NSF AI Institute TILOS.

%} %%%%%%%%%%%%%%%%%%%%%%%%%%%%%%%%%%%%%%%%%%%%%%%%%%%%%PAPER ENDS HERE

%%%%%%%%%%%%%%%%%%%%%%%%%%%%%%%%%%%%%%%%%%%%%%%%%%%%%%%%%%%%

%%%%%%%%%%%%%%%%%%%%%%%%%%%%%%%%%%%%%%%%%%%%%%%%%%%%%%%%%%%%

\newpage

 \bibliography{bibliography}
 \bibliographystyle{bibliography}

\section*{Checklist}

\begin{enumerate}

\item For all authors...
\begin{enumerate}
  \item Do the main claims made in the abstract and introduction accurately reflect the paper's contributions and scope?
    \answerYes{All claims made are backed. up either empirically or theoretically.}
  \item Did you describe the limitations of your work?
    \answerYes{See Appendix \ref{sec: limitations}.}
  \item Did you discuss any potential negative societal impacts of your work?
    \answerYes{See Appendix \ref{sec: broader impact}.}
  \item Have you read the ethics review guidelines and ensured that your paper conforms to them?
    \answerYes{We have read the guidelines, and confirmed that our paper conforms.}
\end{enumerate}

\item If you are including theoretical results...
\begin{enumerate}
  \item Did you state the full set of assumptions of all theoretical results?
    \answerYes{All theoretical result are stated exactly. }
        \item Did you include complete proofs of all theoretical results?
    \answerYes{See Appendix for proofs.}
\end{enumerate}

\item If you ran experiments...
\begin{enumerate}
  \item Did you include the code, data, and instructions needed to reproduce the main experimental results (either in the supplemental material or as a URL)?
    \answerYes{See anonymized URL for all code.}
  \item Did you specify all the training details (e.g., data splits, hyperparameters, how they were chosen)?
    \answerYes{See Appendix \ref{app: unsup comb} and Appendix \ref{app: image classification}}
        \item Did you report error bars (e.g., with respect to the random seed after running experiments multiple times)?
    \answerYes{Except in cases where HPO was run. In these cases we report the test performance of the. model with best validation performance. }
        \item Did you include the total amount of compute and the type of resources used (e.g., type of GPUs, internal cluster, or cloud provider)?
    \answerYes{See Appendix \ref{app: general expt setup}.}
\end{enumerate}

\item If you are using existing assets (e.g., code, data, models) or curating/releasing new assets...
\begin{enumerate}
  \item If your work uses existing assets, did you cite the creators?
    \answerYes{We cite all creators of. existing assets, either in the main paper or appendix.}
  \item Did you mention the license of the assets?
    \answerYes{Yes, see Appendix \ref{app: general expt setup}.}
  \item Did you include any new assets either in the supplemental material or as a URL?
    \answerYes{We provide anonymized code. Open source code will be released after review.}
  \item Did you discuss whether and how consent was obtained from people whose data you're using/curating?
    \answerYes{See Appendix \ref{app: general expt setup}.}
  \item Did you discuss whether the data you are using/curating contains personally identifiable information or offensive content?
    \answerYes{See Appendix \ref{app: general expt setup}.}
\end{enumerate}

\item If you used crowdsourcing or conducted research with human subjects...
\begin{enumerate}
  \item Did you include the full text of instructions given to participants and screenshots, if applicable?
\answerNA{No crowdsourcing or human subjects used.}
  \item Did you describe any potential participant risks, with links to Institutional Review Board (IRB) approvals, if applicable?
\answerNA{No crowdsourcing or human subjects used.}
  \item Did you include the estimated hourly wage paid to participants and the total amount spent on participant compensation?
\answerNA{No crowdsourcing or human subjects used.}
\end{enumerate}

\end{enumerate}

%%%%%%%%%%%%%%%%%%%%%%%%%%%%%%%%%%%%%%%%%%%%%%%%%%%%%%%%%%%% APPENDIX
\appendix

%\section{Appendix}

\newpage

\section{Optimization programs: extended discussion}\label{app: primaldual}
In this section, we provide an extended discussion of the key components of our LP and SDP formulations and the relationships between them. Apart from supplying derivations, another goal of this section is to illustrate that there is in fact flexibility in the exact choice of formulation for the LP (and consequently the SDP). We provide details on possible variations as part of this discussion as a guide to users who may wish to adapt the SFE framework.

\subsection{LP formulation: Derivation of the dual.}First, recall that our primal LP is defined as
\begin{align*}
    \underset{\mathbf{z},b \in \mathbb{R}^n \times \mathbb{R}}{\max}\{ \xb^{\top}\zb +b\} 
    \ \text{ subject to } \  \mathbf{1}_S ^\top \mathbf{z} +b  \leq f(S) \  \text{ for all } S \subseteq [n].  
\end{align*}
The dual is 
\begin{align}
    \underset{\{y_S\geq 0\}_{S \subseteq [n]}}{\min} \sum_{S \subseteq [n]} y_S f(S) \nonumber
    \text{ subject to } \sum_{S \subseteq [n]} y_S \mathbf{1}_S = \xb, \  \sum_{S \subseteq [n]} y_S=1, \ \text{ for all } S\subseteq [n]. 
\end{align} \nonumber

In order to standardize the derivation, we first convert the primal maximization problem into  minimization (this will be undone at the end of the derivation). We have
\begin{align}
     \underset{\mathbf{z},b \in \mathbb{R}^n \times \mathbb{R}}{\min} \{ -\xb^{\top}\zb -b\} \text{ subject to } \  \mathbf{1}_S ^\top \mathbf{z} +b  \leq f(S) \  \text{ for all } S \subseteq [n].  
\end{align}
The Lagrangian is 
\begin{align}
    \underset{y_S \geq 0}{\mathcal{L}(\zb, y_S, b)} &=  -\xb^{\top}\zb -b - \sum_{S\subseteq [n]}y_S(f(S) -  \mathbf{1}_S ^\top \mathbf{z} -b) \\ 
    &= -\sum_{S \subseteq [n]} y_S f(S) +(\sum_{S \subseteq [n]} y_S  \mathbf{1}_S ^\top - \xb^\top) \mathbf{z} +  b(\sum_{S \subseteq [n]} y_S-1)
\end{align}
The optimal solution $\mathbf{p}^*$ to the primal problem is then
\begin{align}
    \mathbf{p}^* &= \underset{\zb, b}{\min{ }} \underset{y_S \geq 0}{\max{ } } \mathcal{L}(\zb, y_S, b) \\
    &= \underset{y_S \geq 0}{\max{ }} \underset{\zb, b}{\min{ } } \mathcal{L}(\zb, y_S, b) \tag{strong duality} \\ 
    &= \mathbf{d}^*,
\end{align}
where $\mathbf{d}^*$ is the optimal solution to the dual.
From the Lagrangian,
\begin{align}
   \underset{\zb, b}{\min{ } } \mathcal{L}(\zb, y_S, b) = \begin{cases} -\sum_{S \subseteq [n]} y_S f(S), \text{  if  } \sum_{S \subseteq [n]} y_S  \mathbf{1}_S  = \xb \text{  and  } \sum_{S \subseteq [n]} y_S=1, \\
   -\infty,  \text{  otherwise}.
    \end{cases}
\end{align}
Thus, we can write the dual problem as 
\begin{align}
    \mathbf{d}^*= \underset{y_S \geq 0}{\max{ }} -\sum_{S \subseteq [n]} y_S f(S) \text{ subject to } \sum_{S \subseteq [n]} y_S  \mathbf{1}_S  = \xb \text{  and  } \sum_{S \subseteq [n]} y_S=1.
\end{align}
Our proposed dual formulation is then obtained by switching from maximization to minimization and negating the objective. It can also be verified that by taking the dual of our dual, the primal is recovered (see \citet[Def. 20]{el2018learning} for the derivation).

\subsection{Connections to submodularity, related linear programs, and possible alternatives.}
Our LP formulation depends on a linear program known to correspond to the convex closure \cite[Eq. 3.57]{murota1998discrete} (convex envelope) of a discrete function. Some readers may recognize the formal similarities of this formulation with the one used to define the Lov\'{a}sz extension \citep{bilmes2022submodularity}. Namely, for $\mathbf{x} \in \mathbb{R}^n$ we can define the Lov\'{a}sz Extension as 
\begin{align}
    \barf(\mathbf{x}) =\underset{ \mathbf{z} \in \mathcal{B}_f}{\text{max }} \mathbf{x}^\top \mathbf{z}   ,
\end{align}
where the feasible set, known as the base polytope of a submodular function, is defined as $\mathcal{B}_f = \{ \mathbf{z} \in \mathbb{R}^n: \mathbf{z}^\top \mathbf{1}_{S} \leq f(S) \; S\subset [n], \text{ and }  \mathbf{z}^\top \mathbf{1}_{S} = f(S) \text{ when } S=[n]  \}$. Base polytopes are also known as \emph{generalized permutahedra} and have rich connections to the theory of matroids, since matroid polytopes belong to the class of generalized permutahedra \cite{ardila2010matroid}.

An alternative option is to consider $\xb \in \mathbb{R}^n_+$, then the Lov\'{a}sz extension is given by
\begin{align}
    \barf(\mathbf{x}) =\underset{ \mathbf{z} \in \mathcal{P}_f}{\text{max }} \mathbf{x}^\top \mathbf{z} ,
\end{align}
where $\mathcal{P}_f$ is the submodular polyhedron as defined in our original primal LP. The subtle differences between those formulations
lead to differences in the respective dual formulations. In principle, those formulations can be just as easily used to define set function extensions. Overall, there are three key considerations when defining a suitable LP:
\begin{itemize}
    \item The constraints of the primal. 
    \item The domain of the primal variables $\zb, b$ and the cost $\xb$.
    \item The properties of the function being extended.
\end{itemize}
Below, we describe a few illustrative example cases for different choices of the above:
\begin{itemize}
    \item Adding the constraint $\mathbf{z}^\top \mathbf{1}_{S} = f(S) \text{ when } S=[n] $ leads to $y_{[n]} \in \mathbb{R}^n$ for the dual. This implies that the coefficients cannot be interpreted as probabilities in general which is what provides the guarantee that the extension will not introduce any spurious minima.  $\sum_{S \subseteq [n]} y_S=1$ is just an affine hull constraint in that case.
    \item For $b=0$, the constraint $\sum_{S \subseteq [n]} y_S=1$ is not imposed in the dual and the probabilistic interpretation of the extension cannot be guaranteed.  Examples that do not rely on this constraint include the homogeneous convex envelope \citep{el2018combinatorial} and the Lov\'{a}sz extension as presented above. However, even for $b=0$, from the definition of the Lov\'{a}sz extension it is easy to see that it retains the probabilistic interpretation when $\xb \in [0,1]$.
    \item Consider a feasible set defined by $\mathcal{P}_f \bigcap \mathbb{R}^n_+$ and let $\xb \in \mathbb{R}^n_+$.  
    If the function $f$ is submodular, non-decreasing and normalized so that $f(\emptyset)=0$ (e.g., the rank function of a matroid), then the feasible set is called polymatroid and $f$ is a polymatroid function. Again, in that case the Lov\'{a}sz extension achieves the optimal objective value \citep[Eq. 44.32]{schrijver2003combinatorial}. In that case, the constraint $\sum_{S \subseteq [n]} y_S  \mathbf{1}_S  = \xb $ of the dual is relaxed to $\sum_{S \subseteq [n]} y_S  \mathbf{1}_S  \geq \xb $. This feasible set of the dual will  allow for more flexible definitions of an extension but it comes at the cost of generality. For instance, for a submodular function that is not non-decreasing, one cannot obtain the Lov\'{a}sz extension as a feasible solution to the primal LP, and the solutions to this LP will not be the convex envelope in general.
\end{itemize}
\subsection{SDP formulation: The geometric intuition of extensions and deriving the dual.}
In order to motivate the SDP formulation, first we have to identify the essential ingredients of the LP formulation.
First, the constraint $\sum_{S \subseteq [n]} y_S  \mathbf{1}_S  = \xb $ captures the simple idea that each continuous point is expressed as a combination of discrete ones, each representing a different set, which is at the core of our extensions. Then, ensuring that the continuous point lies in the convex hull of those discrete points confers additional benefits w.r.t. optimization and offers a probabilistic perspective.

Consider the following example. The Lov\'{a}sz extension identifies each continuous point in the hypercube with a simplex. Then the continuous point is viewed as an expectation over a distribution supported on the simplex corners. The value of the set function at a continuous point is then the expected value of the function over those corners under the same distribution, i.e., $\E_{S \sim p_\xb}[\mathbf{1}_S] = \xb $ leads to $ \mathbb{E}_{S \sim p_\xb} [f(S)]=\barf(\xb)$. As long as the distribution $p_\xb$ can be differentiated w.r.t $\xb$, we obtain an extension that can be used with gradient-based optimization.
It is clear that the construction depends on being able to identify a small convex set of discrete vectors that can express the continuous one.

This can be formulated in higher dimensions, particularly in the space of PSD matrices.  A natural way to represent sets in high dimensions is through rank one matrices that are outer
products of the indicator vectors of the sets, i.e., $\mathbf{1}_S\mathbf{1}_S^\top $ is the matrix representation of $S$ similar to how $\mathbf{1}_S$ is the vector representation. Hence, in the space of matrices, our goal will be again to identify a set of discrete  \emph{matrices} that represents sets that can express a matrix of continuous values.

% Going back to the  Lov\'{a}sz extension  example, a geometric analogue to the simplex in the space of matrices is the {spectraplex}  $\; \mathcal{O}_n$. This is the set of PSD $n \times n$ matrices with trace one \citep{blekherman2012semidefinite}, i.e., 
% \begin{align}
%     \mathcal{O}_n = \{ \mathbf{X} \in \mathbb{S}^n_+ : \text{Tr}(\mathbf{X}=1)\}.
% \end{align}
% Intuitively, the PSD trace one property means that the eigenvalues sum to 1 and can be interpreted as probabilities. Furthermore, each matrix is a convex combination (through the eigenvector expansion) of rank one matrices $\mathbf{vv}^\top$ formed by the corresponding eigenvectors. The goal is then to express each $\mathbf{vv}^\top$  as a combination of matrices that correspond to sets, which is what our dual SDP is achieving through the linear matrix inequality constraint $\mathbf{X}\preceq \frac{1}{2}\sum_{S,T \subseteq [n]} y_{S , T}(\mathbf{1}_{S} \mathbf{1}_{T}^\top+ \mathbf{1}_{T} \mathbf{1}_{S}^\top)$. Here, the products $1$

 The above considerations set the stage for a transition from linear programming to semidefinite programming, where the feasible sets are spectrahedra. Our SDP formulation attempts to capture the intuition described in the previous paragraphs while also maintaining formal connections to the LP by showing that feasible LP regions correspond to feasible SDP regions by simply projecting the LP regions on the space of diagonal matrices (see Proposition \ref{prop:sdp_lp}).

\paragraph{Derivation of the dual.}
Recall that our primal SDP is defined as
 \begin{align}
     \max_{\mathbf{Z} 	\succeq 0, b \in \mathbb{R}} \{\text{Tr}(\mathbf{X^\top Z}) + b\} \text{ subject to } \frac{1}{2}\text{Tr}((\mathbf{1}_S \mathbf{1}_T^\top + \mathbf{1}_T \mathbf{1}_S^\top) \mathbf{Z}) + b \leq f(S\cap T) \text{ for } S,T \subseteq [n].
 \end{align}
 We will show that the dual is
\begin{align}
    &  \min_{ \{y_{S,T}\geq 0 \}}\sum_{ S,
     \subseteq [n]}  y_{S, T}  f(S\cap T) \ \text{ subject to } \ \mathbf{X}\preceq \sum_{S,T \subseteq [n]}\frac{1}{2} y_{S , T}(\mathbf{1}_{S} \mathbf{1}_{T}^\top+ \mathbf{1}_{T} \mathbf{1}_{S}^\top) \ \ \text{ and } \sum_{ S, T \subseteq [n]} y_{{S,T}} = 1.
\end{align}
As before, we convert the primal to a minimization problem:
 \begin{align}
     \max_{\mathbf{Z} 	\succeq 0, b \in \mathbb{R}} \{-\text{Tr}(\mathbf{X^\top Z}) - b\} \text{ subject to } \frac{1}{2}\text{Tr}((\mathbf{1}_S \mathbf{1}_T^\top + \mathbf{1}_T \mathbf{1}_S^\top) \mathbf{Z}) + b \leq f(S\cap T) \text{ for } S,T \subseteq [n].
 \end{align}
 First, we will standardize the formulation by converting the inequality constraints into equality constraints. This can be achieved by adding a positive slack variable $d_{S,T}$ to each constraint such that
 \begin{align}
     \frac{1}{2}\text{Tr}((\mathbf{1}_S \mathbf{1}_T^\top + \mathbf{1}_T \mathbf{1}_S^\top) \mathbf{Z}) + b + d_{S,T} =  f(S\cap T).
 \end{align}
 In matrix notation this is done by introducing the positive diagonal slack matrix $\mathbf{D}$ to the decision variable $\mathbf{Z}$, and extending the symmetric matrices in each constraint
 \begin{align}
     \mathbf{Z}' =
     \begin{bmatrix} \mathbf{Z} & 0 \\
     0 & \mathbf{D}
     \end{bmatrix}, \quad 
          \mathbf{X}' = \begin{bmatrix}
     \mathbf{X} & 0 \\
     0& 0 
     \end{bmatrix}, \quad
     \mathbf{A}_{S,T}' = \begin{bmatrix}\frac{1}{2}(\mathbf{1}_S \mathbf{1}_T^\top + \mathbf{1}_T \mathbf{1}_S^\top) & 0 \\
     0 & \text{diag}(\mathbf{e}_{S,T})
     \end{bmatrix},
 \end{align}
 where $\text{diag}(\mathbf{e}_{S,T})$ is a diagonal matrix where all diagonal entries are zero except at the diagonal entry corresponding to the constraint on $S,T$ which has a $1$.
 Using this reformulation, we obtain an equivalent SDP in standard form:
  \begin{align}
     \max_{\mathbf{Z'} 	\succeq 0, b \in \mathbb{R}} \{-\text{Tr}(\mathbf{X'^\top Z'}) - b\} \text{ subject to } \text{Tr}( \mathbf{A}_{S,T}' \mathbf{Z}') + b = f(S\cap T) \text{ for } S,T \subseteq [n].
 \end{align}
 Next, we form the Lagrangian which features a decision variable $y_{S,T}$ for each inequality, and a dual matrix variable $\boldsymbol{\Lambda}$.
 We have 
 \begin{align}
     \mathcal{L}(\mathbf{Z}',b, y_{S,T}, \boldsymbol{\Lambda}) &= -\text{Tr}(\mathbf{X'^\top Z'}) - b - \sum_{ S,T
     \subseteq [n]}  y_{S, T}\left(2f(S\cap T) -  \text{Tr}( \mathbf{A}_{S,T}' \mathbf{Z}') - b\right) - \text{Tr}(\boldsymbol{\Lambda}\mathbf{Z'}) \\
     &= \text{Tr}\left(((\sum_{ S,T
     \subseteq [n]}  y_{S, T} \mathbf{A}_{S,T}') - \mathbf{X'} - \boldsymbol{\Lambda} )\mathbf{Z}' \right) + b(\sum_{ S,T
     \subseteq [n]} y_{S, T} - 1)  - \sum_{ S,T
     \subseteq [n]}  y_{S, T} f(S \cap T)
 \end{align}
 For the solution to the primal $\mathbf{p}^*$, we have
 \begin{align}
    \mathbf{p}^* &= \underset{\mathbf{Z}',b}{\min} \underset{\boldsymbol{\Lambda}, y_{S,T}}{\max} \mathcal{L}(\mathbf{Z}',b, y_{S,T},\boldsymbol{\Lambda}) \\
    &\geq   \underset{\boldsymbol{\Lambda}, y_{S,T}}{\max } \tag{weak duality} \underset{\mathbf{Z}',b}{\min{ } }  \mathcal{L}(\mathbf{Z}',b, y_{S,T},\boldsymbol{\Lambda}) \\
     &=\mathbf{d}^*.
 \end{align}
 For our Lagrangian we have the dual function 
 \begin{align}
      \underset{\mathbf{Z'},b}{\min{ } }  \mathcal{L}(\mathbf{Z}',b,y_{S,T},\boldsymbol{\Lambda}) = \begin{cases}
      0, \text{ if } \boldsymbol{\Lambda} \succeq 0, \\
      -\infty, \text{ otherwise }.
      \end{cases}
 \end{align}
 Thus, the dual function $ \underset{\mathbf{Z}',b}{\min{ } }  \mathcal{L}(\mathbf{Z}',b,y_{S,T},\boldsymbol{\Lambda})$ takes non-infinite values under the conditions
\begin{align}
  ( \sum_{ S,T
     \subseteq [n]}  y_{S, T} \mathbf{A}_{S,T}') - \mathbf{X'} - \boldsymbol{\Lambda}  &= 0, \\
     \boldsymbol{\Lambda} & \succeq 0, \\
      \text{  and } \sum_{ S,T
     \subseteq [n]} y_{S, T} - 1  &= 0.
\end{align}
The first two conditions imply the linear matrix inequality (LMI)
\begin{align}
      \sum_{ S,T
     \subseteq [n]}  y_{S, T} \mathbf{A}_{S,T}' - \mathbf{X'} \succeq 0 \tag{$\boldsymbol{\Lambda} \succeq 0$}.
\end{align}
From the definition of $\mathbf{A}_{S,T}'$ we know that its additional diagonal entries will correspond to the variables $y_{S,T}$. Combined with the conditions above, we arrive at the constraints of the dual
\begin{align}
    y_{S,T} & \geq 0, \\
  \sum_{S,T \subseteq [n]} \frac{1}{2}y_{S , T}(\mathbf{1}_{S} \mathbf{1}_{T}^\top+ \mathbf{1}_{T} \mathbf{1}_{S}^\top) & \succeq  \mathbf{X}, \\
  \sum_{ S,T \subseteq [n]} y_{S, T}   & = 1.
\end{align}
This leads us to the dual formulation
\begin{align}
    \underset{y_{S,T}\geq 0}{\max{}} - \sum_{ S,T
     \subseteq [n]} y_{S, T} f(S \cap T)  \text{  subject to  } 
      \sum_{S,T \subseteq [n]} \frac{1}{2}y_{S , T}(\mathbf{1}_{S} \mathbf{1}_{T}^\top+ \mathbf{1}_{T} \mathbf{1}_{S}^\top) & \succeq  \mathbf{X} \text {  and  }\sum_{ S,T \subseteq [n]} y_{S, T}   = 1.
\end{align}
Then, we can obtain our original dual by switching to minimization and negating the objective.

\section{Scalar Set Function Extensions Have No Bad Minima}\label{app: vector SFE proofs}
In this section we re-state and prove the results from Section \ref{sec: scalar SFEs}. The first result concerns the minima of $\barf$, showing that the minimum value is the same as that of $f$, and no additional minima are added (besides convex combinations of discrete minimizers). These properties are especially desirable when using an extension $\barf$ as a loss function (see Section \ref{sec: experiments}) since it is important that $\barf$ drive the neural network $\text{NN}_1$ towards producing discrete $\mathbf{1}_S$ outputs.

\begin{proposition}[Scalar SFEs have no bad minima]
If $\barf$ is a scalar SFE of $f$ then:
\begin{enumerate}
    \item $\min_{\xb\in  \calX} \barf(\xb)=\min_{S \subseteq [n]} f(S) $
    \item %The minima of $ \barf(\xb)$ over $\xb\in [0,1]^d$ are a subset of the convex hull of  the minima of $ f(\xb)$ over $\xb\in \Omega$ 
    $\argmin_{\xb\in \calX} \barf(\xb) \subseteq \text{Hull} \big (\argmin_{\mathbf{1}_S : S \subseteq [n]} f(S) \big )$ 
\end{enumerate}
\end{proposition}

%The assumptions on $f$ are  mild.Since we assume without loss of generality that $f(\emptyset)=0$, the non-negativity condition $\min_{S \subseteq [n]} f(S) < 0$  merely asserts that the minimization problem does not have the trivial solution $S=\emptyset$.

\begin{proof}
The inequality $\min_{\xb\in \calX} \barf(\xb) \leq \min_{S \subseteq [n]} f(S) $ automatically holds since $\min_{S \subseteq [n]} f(S) = \min_{\mathbf{1}_S : S \subseteq [n]} \barf(\mathbf{1}_S)$, and $\{\mathbf{1}_S : S \subseteq [n]\} \subseteq \calX$. So it remains to show the reverse. Indeed, letting $\xb \in \calX$ be an arbitrary point we have,
\begin{align*}
     \barf (\xb) &=  \mathbb{E}_{S \sim p_\xb} [f(S)] \\
     &=\sum_{S \subseteq [n]} p_\xb (S) \cdot f(S) \\
     &\geq \sum_{S \subseteq [n]} p_\xb (S) \cdot \min_{S \subseteq [n]} f ( S ) \\
     &=\min_{S \subseteq [n]} f ( S )
\end{align*} 
where the last equality simply uses the fact that $\sum_{S \subseteq [n]} p_\xb (S)=1$. This proves the first claim.

To prove the second claim, suppose that $\xb$ minimizes  $ \barf(\xb)$ over $\xb\in \calX$. This implies that the inequality  in the above derivation must be tight, which is true if and only if
\begin{align*}
p_\xb (S) \cdot  f(S) =  p_\xb (S) \cdot \min_{S \subseteq [n]} f ( S )  \quad  \text{ for all } S \subseteq [n].
\end{align*}
 For a given $S$, this implies that either $p_\xb (S)=0$ or $ f(S) = \min_{S \subseteq [n]} f ( S ) $. Since  $\xb = \mathbb{E}_{p_\xb}[ \vone_S] =  \sum_{S \subseteq [n]} p_\xb(S)\cdot \vone_S =\sum_{S: p_\xb(S) >0} p_\xb(S)\cdot \vone_S$. This is precisely a convex combination of points $\vone_S$ for which $f(S) =  \min_{S \subseteq [n]} f(S)$. Since $\barf$ is a convex combination of exactly this set of points $\vone_S$, we have the second claim.
 
 \iffalse
 
To prove that $\xb$ belongs to the convex hull of  the minimizers of $ f(\xb)$ over $\xb\in \Omega$  we must show that there are coefficients $a_S $ such that $\xb = \sum_{S \in \Omega} a_S S$ and have the properties: $a)$ $\sum_{S \in \Omega} a_S=1$, $b)$  $S \in \arg \min _{\xb \in \Omega} f(\xb)$ for all $S$ such that $a_S \neq 0$.

Property $a)$ follows since $( \sum_i a_i -1  ) \cdot \min_{\xb \in \Omega} f(\xb) = 0 $ and $ \min_{\xb \in \Omega} f(\xb) < 0$. Property $b)$ is true since $a_i \cdot  f(S_i) = a_i \cdot \min_{\xb\in \Omega} f(\xb) $ implies that either $a_i=0$ or $ f(S_i) = \min_{\xb\in \Omega} f(\xb) $.
\fi
\end{proof}

\section{Examples of Vector Set Function Extensions}\label{app: vector SFE examples}

This section re-defines the vector SFEs given in Section \ref{sec: constructing scalar SFEs}, and prove that they satisfy the definition of an SFEs. One of the conditions we must check is that $\barf$ is continuous. A sufficient condition for continuity (and almost everywhere differentiability) that we shall use for a number of constructions is to show that $\barf$ is Lipschitz. A very simple computation shows that it suffices to show that $\xb \in \calX  \mapsto p_\xb(S)$ is Lipschitz continuous.

\begin{lemma}\label{prop: lipschitz}
If the mapping $\xb \in [0,1]^n \mapsto p_\xb(S)$ is Lipschitz continuous and $f(S)$ is finite for all $S$ in the support of $p_\xb$, then $\barf$ is also Lipschitz continuous. In particular, $\barf$ is continuous and almost everywhere differentiable. 
\end{lemma}

\begin{proof}
The Lipschitz continuity of  $\barf (\xb)$ follows directly from definition:
\begin{align*}
  \big |  \barf(\xb) -  \barf(\xb') \big |  &=  \bigg |\sum_{S \subseteq [n]} p_\xb(S)\cdot f(S) - \sum_{S \subseteq [n]} p_{\xb'}(S)\cdot f(S) \bigg |\\
    &=\bigg | \sum_{S \subseteq [n]} \big (p_\xb(S) - p_{\xb'}(S)\big ) \cdot f(S)  \bigg | 
    % &\leq  \sum_{S \in \Omega} \lvert p_\xb(S) - p_{\xb'}(S) \rvert \cdot \max_{S \in \Omega} f(S)  \\
    \leq  \left(2 k L \max_{S \subseteq [n]} f(S) \right) \, \cdot \| \xb - \xb' \|,  
    % =  | \Omega |\cdot L  \cdot \max_{S \in \Omega} f(S) \cdot \| \xb - \xb' \|,
\end{align*}
where $L$ is the maximum Lipschitz constant of $\xb \mapsto p_\xb(S)$ over any $S$ in the support of $p_\xb$, and $k$ is the maximal cardinality of the support of any $p_\xb$. % (note that since $\Omega$ is finite we may take $L$ to be sufficiently large that it is the Lipschitz constant for all $S \in \Omega$.
\end{proof}

In general $k$ can be trivially bounded by $2^n$, so $\barf$ is always Lipschitz. However in may cases the cardinality of the support of any $p_\xb$ is much smaller than $2^n$, leading too a smaller Lipschitz constant. For instance, $k=n$ in the case of the Lov\'{a}sz  extension.

\subsection{Lov\'{a}sz  extension.}\label{app: lovasz}

Recall the definition: $\xb$ is sorted so that $x_1 \geq x_2 \geq \ldots \geq x_d$. Then the Lov\'{a}sz  extension corresponds to taking $S_i = \{1,\ldots , i\}$, and letting $p_\xb({S_i})= x_i - x_{i+1}$, the non-negative increments of $\xb$ (where recall we take $x_{n+1}=0$). All other sets have zero probability. For convenience, we introduce the shorthand notation $a_i =p_\xb(S_i) = x_i - x_{i+1}$

\paragraph{Feasibility.} Clearly all $a_i = x_i - x_{i+1} \geq 0$, and $\sum_{i=1}^n a_i = \sum_{i=1}^n (x_i - x_{i+1}) = x_1 \leq 1$. Any remaining probability mass is assigned to the empty set: $p_\xb(\emptyset) = 1-x_1$, which contributes nothing to the extension $\barf$ since $f(\emptyset)=0$ by assumption. All that remains is to check that
$$\sum_{i=1}^n p_\xb(S_i)\cdot \vone_{S_i}= \xb.$$
For a given $k \in [n]$, note that the only sets $S_i$ with non-zero $k$th coordinate are $S_1, \ldots , S_k$, and in all cases $(\vone_{S_i})_k=1$. So the $k$th coordinate is precisely $\sum_{i=1}^k p_\xb(S_i)= \sum_{i=1}^k (x_i - x_{i+1})=x_k$, yielding the desired formula. 

\paragraph{Extension.} Consider an arbitrary $S \subseteq [n]$. Since we assume $\xb= \vone_S$ is sorted, it has the form $\vone_S= ( \underbrace{ 1,1, \ldots , 1}_{k \text{ times}}, 0, 0, \ldots 0)^\top$. Therefore, for each $j < k$ we have $a_j =  x_j- x_{j+1} = 1-1=0 $ and for each $j > k$ we have $a_j =  x_j - x_{j+1} = 0-0=0 $. The only non-zero probability is $a_k =  x_k - x_{k+1} = 1-0 =1$. So, 
$$\barf(\vone_S) = \sum_{i=1}^n a_i f(S_i) =  \sum_{i : i\neq k} a_i f(S_i) + a_k f(S_k) = 0 + 1\cdot f(S_k) = f(S)$$
where the the final equality follows since by definition $S_k$ corresponds exactly to the vector $( \underbrace{ 1,1, \ldots , 1}_{k \text{ times}}, 0, 0, \ldots 0)^\top= \vone_S$ and so $S_k=S$.

\paragraph{Continuity.}

The Lov\'{a}sz is a well-known extension, whose properties have been carefully studied. In particular it is well known to be a Lipschitz function \cite{bach2019submodular}. However, for completeness we provide a simple proof here nonetheless.

 \begin{lemma}
 Let $p_\xb$ be as defined for the Lov\'{a}sz  extension. Then $\xb \mapsto p_\xb(S)$ is Lipschitz for all $S \subseteq [n]$.
 \end{lemma}
 
 \begin{proof}

 First note that $p_\xb$ is piecewise linear, with one piece per possible ordering  $x_1 \geq x_2 \geq \ldots \geq x_n$ (so $n!$ pieces in total). Within the interior of each piece $p_\xb$ is linear, and therefore Lipschitz. So in order to prove global Lipschitzness, it suffices to show that $p_\xb$ is continuous at the boundaries between pieces (the  Lipschitz constant is then the  maximum of the Lipschitz constants for  each linear piece). 
 
 Now consider a point $\xb$ with   $x_1 \geq  \ldots\geq x_i = x_{i+1} \geq \ldots \geq x_n$. Consider the perturbed point $\xb_\delta = \xb - \delta \mathbf{e}_i$ with $\delta>0$, and $\mathbf{e}_i$ denoting the $i$th standard basis vector. To prove continuity of  $p_\xb$  it suffices  to show that for any $S \in \Omega$ we have $p_{\xb_\delta}(S) \rightarrow p_\xb(S)$ as $\delta \rightarrow 0^+$. 
 
 There are two sets in the support of $p_\xb$ whose probabilities are different under $p_{\xb_\delta}$, namely:  $S_i = \{ 1,\ldots , i\}$ and $S_{i+1} = \{ 1,\ldots , i,  i+1\}$. Similarly, there are  two sets in the support of  $p_{\xb_\delta}$ whose probabilities are different under $p_\xb$, namely: $S'_i = \{ 1,\ldots , i-1, i+1\}$ and $S'_{i+1} = \{ 1,\ldots ,  i,  i+1\} = S_{i+1}$.
 So it suffices to show the convergence $p_{\xb_\delta}(S) \rightarrow p_\xb(S)$ for these four $S$. Consider first $S_i$: 
 $$\big |p_{\xb_\delta}(S_i) - p_\xb(S_i) \big | = \big | 0 - (x_i - x_{i+1}) \big | = 0 $$
 where the final equality uses the fact that $x_i = x_{i+1}$. Next consider $S_{i+1} = S'_{i+1}$: 
  $$\big |p_{\xb_\delta}(S_{i+1}) - p_\xb(S_{i+1}) \big | = \big |  (x'_{i+1} - x'_{i+2})  - (x_{i+1} - x_{i+2}) \big | = \big |  (x'_{i+1} - x_{i+1})  - (x'_{i+2} - x_{i+2}) \big | = 0 $$
  Finally, we consider $S'_{i}$:
   \begin{align*}
       \big |p_{\xb_\delta}(S'_{i}) - p_\xb(S'_{i}) \big | &= \big |  (x'_{i} - x'_{i+1})  - (x_{i} - x_{i+1}) \big | \\
       &= \big |  (x'_{i+1} - x_{i+1})  - (x'_{i+1} - x_{i+1}) \big | \\ 
       &=  \big |  (x_{i+1} - \delta - x_{i+1})  - (x'_{i+1} - x_{i+1}) \big |\\
       &= \delta \rightarrow 0
    \end{align*}
    
    completing the proof.
     \end{proof}

\subsection{Bounded cardinality  Lova\'{s}z extension.}The bounded cardinality extension coefficients $p_\xb(S)$  are the coordinates of the vector $\mathbf{y}$, where $\mathbf{y}= \mathbf{S}^{-1}\mathbf{x}$ and the entries $(i,j)$ of the inverse are
 \begin{align}
     \mathbf{S}^{-1}(i,j) = \begin{cases}
       1, \text{  if }  (j-i) \text { mod } k = 0 \text{ and } i\leq j, \\
       -1, \text{  if }  (j-i) \text { mod } k = 1 \text{ and } i\leq j, \\
       0, \text { otherwise}.
     \end{cases}
 \end{align}

\paragraph{Equivalence to the  Lova\'{s}z extension.}
We want to show that the bounded cardinality extension is equivalent to the  Lova\'{s}z extension when $k=n$.
 Let $T_{i,k}= \{j \; | \; (j-i) \text{ mod } k =0, \text{ for } i\leq j\leq n, \; j \in \mathbb{Z}_+\}$, i.e., $T_{i,k}$ stores the indices where $j-i$ is perfectly divided by $k$. From the analytic form of the inverse, observe that the $i$-th  coordinate of $\mathbf{y}$ is $p_\xb (S_i)= \sum_{j \in T_{i,k}} (x_j-x_{j+1})$. For $k=n$, we have $T_{i,n}=  \{j \; | \; (j-i) \text{ mod } n =0 \} =\{ i \}$, and therefore $p_\xb (S_i)=  x_i-x_{i+1}$, which are the coefficients of the Lov\'{a}sz extension.

 \paragraph{Feasibility.} 
 The equation $\mathbf{y}= \mathbf{S}^{-1}\mathbf{x}$ guarantees that the constraint $\mathbf{x} = \sum_{i=1}^n y_{S_i} \mathbf{1}_{S_i}$ is obeyed.
Recall that $\mathbf{x}$ is sorted in descending order like in the case of the Lov\'{a}sz extension. Then, it is easy to see that $p_\xb (S_i)= \sum_{j \in T_{i,k}} (x_j-x_{j+1})\leq x_i$, because $x_i-x_{i+1}$ is always contained in the summation for $p_\xb (S_i)$. Therefore, by restricting $\mathbf{x}$ in the probability simplex it is easy to see that $\sum_{i=1}^n p_\xb (S_i) \leq \sum_{i=1}^n x_i = 1 $. To secure tight equality, we allocate the rest of the mass to the empty set, i.e., $p_\xb(\emptyset)= 1-\sum_{i=1}^n p_\xb (S_i)$, which does not affect the value of the extension sicne the corresponding Boolean is the zero vector.
 
\paragraph{Extension.}
To prove the extension property we need to show that $\barf(\mathbf{1}_S) = f(S)$ for all $S$ with $|S|\leq k$.  Consider any such set $S$ and recall that we have sorted $\mathbf{1}_S$ with arbitrary tie breaks, such that $x_i=1$ for $i\leq |S|$ and $x_i=0$ otherwise. Due to the equivalence with the 
 Lova\'{s}z extension, the extension property is guaranteed when $k=n$ for all possible sets. For $k < n$,  consider the following three cases for $T_{i,k}$. 
 \begin{itemize}
     \item When $i>|S|$, $T_{i,k} = \emptyset$ because for sorted $\mathbf{x}$ of cardinality at most $k$, we know for the coordinates that $x_i=x_{i+1}=0$. For $i>k$, this implies that $p_\xb(S_i)=0$.
     \item When $i<|S|$, $\sum_{j \in T_{i,k}} (x_j-x_{j+1}) = 0$ because $x_j=x_{j+1}=1$ and we have again $p_\xb(S_i)=0$.
     \item When $i=|S|$, observe that $ \sum_{j \in T_{i,k}} (x_j-x_{j+1}) = x_i-x_{i+1} = x_i$. Therefore, $p_\xb(S_i)=1.$ in that case.  
 \end{itemize}
 Bringing it all together, $\barf(\mathbf{1}_S)= \sum_{i=1}^n p_\xb f(S_i)=  p_\xb(S) f(S) = f(S)$ since the sum contains only one nonzero term, the one that corresponds to $i=|S|$.
 
\paragraph{Continuity.}
Similar to the Lova\'{s}z extension, $p_\xb$ in the bounded cardinality extension is piecewise linear and therefore a.e. differentiable with respect to $\mathbf{x}$, where each piece corresponds to an ordering of the coordinates of $\mathbf{x}$. On the other hand, unlike the Lova\'{s}z extension, the mapping $\mathbf{x} \mapsto p_\xb(S)$ is not necessarily globally Lipschitz when $k < n$, because it is not guaranteed to be Lipschitz continuous at the boundaries.

\subsection{Singleton extension.}

 \paragraph{Feasibility.}

The singleton extension is not dual LP feasible. However, one of the key reasons why feasibility is important is that it implies Proposition \ref{prop: nice properties of extension}, which show that optimizing $\barf$ is a reasonable surrogate to $f$. In the case of the singleton extension, however, Proposition \ref{prop: nice properties of extension} still holds even without feasibility for $f$. This includes the case of the training accuracy loss, which can be viewed as minimizing the set function $f(\{\hat{y}\}) = -\mathbf{1}\{ y_i =\hat{y} \}$.

 Here  we give an alternative proof of Proposition \ref{prop: nice properties of extension} for the singleton extension. Consider the same  assumptions as Proposition \ref{prop: nice properties of extension} with the additional requirement that $\min_S f(S) < 0 $ (this merely asserts hat $S = \emptyset$ is not a trivial solution to the minimization problem, and that the minimizer of $f$ is unique. This is true, for example, for the training accuracy objective we consider in Section \ref{sec: experiments}.

\begin{proof}[Proof of Proposition \ref{prop: nice properties of extension} for singleton extension] For  $\xb \in \calX = [0,1]^n$,
\begin{align*}
    \barf(\xb) &= \sum_{i=1}^n p_\xb(S_i) f(S_i) \\
    &= \sum_{i=1}^n (x_i - x_{i+1}) f(S_i) \\
    &\geq  \sum_{i=1}^n (x_i - x_{i+1})\min_{j \in [n]} f(S_j) \\
    &\geq  (x_1 - x_{n+1})\min_{j \in [n]} f(S_j) \\
    &\geq  x_1 \cdot \min_{j \in [n]}f(S_j) \\
     &\geq  \min_{j \in [n]}f(S_j) \\
\end{align*}
where the final inequality follows since $\min_{j \in [n]}f(S_j)<0$. Taking $\xb=(1, 0, 0, \ldots , 0)^\top$ shows that all the inequalities can be made tight, and the first statement of Proposition \ref{prop: nice properties of extension} holds. For the second statement, suppose that  $\xb \in \calX = [0,1]^n$ minimizes $\barf$. Then all the inequality in the preceding argument must be tight. In particular, tightness of the final inequality implies that $x_1=1$. Meanwhile, tightness of the first inequaliity implies that $ x_i - x_{i+1} = 0$ for all $i$ for which  $f(S_i) \neq \min_{j \in [n]} f(S_j)$, and tightness of the second inequality implies that $x_{n+1}=0$. These together imply that $\xb = \mathbf{1}    \oplus \mathbf{0}_{n-1}$  where $\mathbf{1} $ is a $1 \times 1$  vector with entry equal to one, and  $\mathbf{0}_{n-1}$ is an all zeros vectors of length $n-1$, and $ \oplus$ denotes concatenation. Since $f(S_1) =\min_{j \in [n]} f(S_j)$ is the unique minimize we have that $\xb  = \mathbf{1}_{S_1} \in \text{Hull} \big (\argmin_{\mathbf{1}_{S_i} : i \in [n]} f(S_i) \big )$, completing the proof.
\end{proof}

\paragraph{Extension.}

Consider an arbitrary $i \in [n]$. Since we assume $\xb= \vone_{\{i\}}$ is sorted, we are without loss of generality considering $ \vone_{\{1\}}= (1, 0 , \ldots , 0, 0, \ldots 0)^\top$. Therefore,  we have $p_\xb (S_1) =  x_1- x_{2} = 1-0=1 $ and for each $j > 1$ we have $p_\xb (S_j) =  x_j - x_{j+1} = 0-0=0 $. The only non-zero probability is  $p_\xb (S_1)$, and so
$$\barf(\vone_{\{1\}}) = \sum_{j=1}^n p_\xb(S_j) f(S_j)= f(S_1) = f(\{1\}).$$

\paragraph{Continuity.} The proof of continuity of the singleton extension is a simple adaptation of the proof used for the Lova\'{s}z extension, which we omit.
 
\subsection{Permutations and Involutory Extension.} 
 \paragraph{Feasibility.}
 It is known that every elementary permutation matrix is involutory, i.e., $\mathbf{SS=I}$. Given such an elementary permutation matrix $\mathbf{S}$, since $\mathbf{S(S} \mathbf{x}) = \mathbf{S}p_\xb= \mathbf{x}$, the constraint $\sum_{S \subseteq [n]} y_S  \mathbf{1}_S  = \xb$ is satisfied. Furthermore,  $\sum_{S \subseteq [n]} y_S=1$ can be secured if $\xb$ is in the simplex, since the sum of the elements of a vector is invariant to permutations of the entries.
 
\paragraph{Extension.}
If the permutation has a fixed point at the maximum element of $\mathbf{x}$, i.e., it maps the maximum element to itself, then any elementary permutation matrix with such a fixed point yields an extension on singleton vectors. Without loss of generality, let $\xb = \mathbf{e}_1$, where $\mathbf{e}_1$ is the standard basis vector in $\mathbb{R}^n$. Then $\mathbf{Se}_1 = \mathbf{e}_1$ and therefore $p_\xb(\mathbf{e}_1) = 1 $. This in turn implies $\barf(\mathbf{e}_1)=1\cdot f(\mathbf{e}_1)$. This argument can be easily applied to all singleton vectors.

\paragraph{Continuity.}
The permutation matrix $\mathbf{S}$ can be chosen in advance for each $\xb$ in the simplex. Since $p_\xb = \mathbf{S}\xb$, the probabilities are piecewise-linear and each piece is determined by the fixed point induced by the maximum element of $\xb$. Consequently, $p_\xb$ depends continuously on $\xb$.

\subsection{Multilinear extension.} 

Recall that the multiliniear extension is defined via $p_\xb(S)= \prod_{i\in S }x_i \prod_{i\notin S }(1-x_i)$ supported on  all subsets $S \subseteq [n]$
 in general.
 \paragraph{Feasibility.}

 The definition of  $p_\xb(S)$ is equivalent to:
 $$p_\xb(S)= \prod_{i=1 }^n x_i^{y_i} (1-x_i)^{1-y_i}$$
 where $y_i=1$ if $i \in S$ and zero otherwise. That is, $p_\xb(S)$ is the  product of $n$ independent Bernoulli distributions. So we clearly have $p_\xb(S)\geq 0$ and $\sum_{S \subseteq [n]} p_\xb(S)= 1$. The final feasibility condition, that $\sum_{S \subseteq [n]} p_\xb(S)\cdot \vone_{S}= \xb$ can be checked by induction on $n$. For $n=1$ there are only two sets: $\{1\}$ and the empty set. And clearly $p_\xb(\{1\})\cdot \vone_{\{1\}} = x_1 (1-x_1)^{0}= x_1$, so we have the base case.

\paragraph{Extension.}

For any $S \subseteq [n]$ we have $p_{\vone_S}(S) = \prod_{i\in S }x_i \prod_{i\notin S }(1-x_i) = \prod_{i\in S }1 \prod_{i\notin S }(1-0) = 1$. So $\barf(\vone_S) = \mathbb{E}_{T \sim p_\xb} f(T) = f(S)$.
\paragraph{Continuity.}
Fix and $S \subseteq [n]$. Again we check Lipschitzness. We use $\partial_{x_k}$ to denote the derivative operator with respect to $x_k$. If $k \in S$ we have 
$$ \abs{\partial_{x_k} p_{\vone_S}(S)} = \abs{ \partial_{x_k}\prod_{i\in S }x_i \prod_{i\notin S }(1-x_i)} = \prod_{i\in S\setminus \{k\} }x_i \prod_{i\notin S }(1-x_i) \leq 1.$$
Similarly, if $k \notin S$ we have,
$$ \abs{\partial_{x_k} p_{\vone_S}(S)} = \abs{ \partial_{x_k}\prod_{i\in S }x_i \prod_{i\notin S }(1-x_i)} = \abs{ - \prod_{i\in S }x_i \prod_{i\notin S\cup \{k\} }(1-x_i)} \leq 1.$$
Hence the spectral norm of the Jacobian $J p_\xb(S)$ is bounded, and so $\xb \mapsto p_\xb(S)$ is a Lipschitz map.

\section{Neural Set Function Extensions}\label{app: neural SFE proofs}

This section re-states and proves the results from Section \ref{sec: neural SFEs}. To start, recall the definition of the primal LP:
\begin{align*}
     \underset{\mathbf{z},b}{\max}\{ \xb^{\top}\zb +b\}, 
    \ \text{ where } \   (\mathbf{z},b) \in \mathbb{R}^n \times \mathbb{R} \text{ and } \mathbf{1}_S ^\top \mathbf{z} +b  \leq f(S) \  \text{ for all } S \subseteq [n].  
\end{align*}
and primal SDP:
 \begin{align}
     \max_{\mathbf{Z} 	\succeq 0, b \in \mathbb{R}} \{\text{Tr}(\mathbf{X^\top Z}) + b\} \text{ subject to } \frac{1}{2}\text{Tr}((\mathbf{1}_S \mathbf{1}_T^\top + \mathbf{1}_T \mathbf{1}_S^\top) \mathbf{Z}) + b \leq f(S\cap T) \text{ for } S,T \subseteq [n].
 \end{align}
\begin{proposition*}
(Containment of LP in SDP) For any $\xb \in [0,1]^n$, 
define $\mathbf{X} = \sqrt{\xb} \sqrt{\xb}^\top$ with the square-root taken entry-wise.  Then, for any $(\zb,b) \in \mathbb{R}^n_+ \times \mathbb{R}$ that is primal LP feasible, the pair  $(\mathbf{Z} ,b)$ where $\mathbf{Z}=\text{diag}(\zb)$, is primal SDP feasible and the objective values agree:  $\text{Tr}(\mathbf{X^\top Z}) =\zb^\top \xb$.
\end{proposition*}

\begin{proof}
We start with the feasibility claim. Suppose that $(\zb,b) \in \mathbb{R}^n_+ \times \mathbb{R}$ is a feasible solution to the primal LP. We must show that $(\mathbf{Z} ,b)$ is a feasible solution to the primal SDP with $\mathbf{X} = \sqrt{\xb} \sqrt{\xb}^\top$ and where $\mathbf{Z}=\text{diag}(\zb)$.

Recall the general formula for the trace of a matrix product: $\text{Tr}(\mathbf{AB}) = \sum_{i,j} A_{ij}B_{ji}$. With this in mind, and noting that the $(i,j)$ entry of $\mathbf{1}_S \mathbf{1}_T^\top$ is equal to $1$ if $i,j \in S\cap T$, and zero otherwise, we have  for any $S,T \subseteq [n]$ that 
\begin{align*}
   \frac{1}{2}\text{Tr}((\mathbf{1}_S \mathbf{1}_T^\top + \mathbf{1}_T \mathbf{1}_S^\top) \mathbf{Z}) =  \text{Tr}(\mathbf{1}_S \mathbf{1}_T^\top \mathbf{Z}) + b &= \sum_{i,j=1}^n (\mathbf{1}_S \mathbf{1}_T^\top)_{ij} \cdot \text{diag}(\zb)_{ij} + b\\
    &=\sum_{i,j \in S\cap T} (\mathbf{1}_S \mathbf{1}_T^\top)_{ij} \cdot \text{diag}(\zb)_{ij} + b \\
     &=\sum_{i,j \in S\cap T}  \text{diag}(\zb)_{ij} + b \\
      &=\sum_{i \in S\cap T}  z_i + b \\
       &=\vone_{S\cap T} ^\top \zb + b \\
       &\leq f(S\cap T)
\end{align*}
showing SDP feasibility. That the objective values agree is easily seen since:
\begin{align*}
\text{Tr}(\mathbf{ZX}) =  \sum_{i,j=1}^n \text{diag}(\zb)_{ij} \cdot \sqrt{x_i} \sqrt{x_j}  =\sum_{i=1}^n z_i \cdot \sqrt{x_i} \sqrt{x_i}  = \xb^{\top}\zb.
\end{align*}
\end{proof}

Next, we provide a proof for the construction of neural extensions.
Recall the statement of the main result. 
\begin{proposition*}
Let $p_\xb$ induce a scalar SFE of $f$. For $\mathbf{X} \in \mathbb{S}_+^n$ with distinct eigenvalues, consider the decomposition $\mathbf{X} = \sum_{i=1}^n \lambda_i \xb_i \xb_i^\top$ and fix
\begin{align}
  p_\mathbf{X}(S, T) = \sum_{i=1}^n \lambda_i \, p_{\xb_i}(S) p_{\xb_i}(T)   \text{ for all } S,T \subseteq [n]. 
\end{align}
Then, $ p_\mathbf{X}$ defines a neural SFE $\barf$ at $\mathbf{X}$.
\end{proposition*}

\begin{proof} 
We begin by showing through the eigendecomposition of $\mathbf{X}$ that the  $\barf$ defined by  $ p_\mathbf{X}(S, T) $ is dual SDP feasible. It is clear that $\sum_{S,T}  p_\mathbf{X}(S, T) = 1$ as long as $\sum_{i=1}^n \lambda_i =1$, which can be easily enforced by appropriate normalization of $\mathbf{X}$. Recall  from the eigendecomposition we have $\mathbf{X}=\sum_{i=1}^{n} \lambda_i \mathbf{v}_i\mathbf{v}_i^\top$ where we have fixed each $\mathbf{v}_i \in [0,1]^n$ through a sigmoid. Using the scalar SFE $p_\xb$ we may write each $\mathbf{v}_i$ as a convex combination $\mathbf{v}_i = \sum_S p_{\mathbf{v}_i}(S) \mathbf{1}_S$. %For simplicity we re-index the sum to be over sets indexed by $j$ for each $i$: $\mathbf{v}_i = \sum_{j} p_{\mathbf{v}_i}(S_{ij}) \mathbf{1}_{S_{ij}}$. 
For each $i$ we may use this representation to re-express the outer product of $\mathbf{v}_i$ with itself:
\begin{align*}
    \mathbf{v}_i \mathbf{v}_i^{\top} &=\big (  \sum_S p_{\mathbf{v}_i}(S) \mathbf{1}_S \big ) \big (  \sum_T p_{\mathbf{v}_i}(T) \mathbf{1}_T \big )^\top \\
    &= {\sum_{S}p_{\mathbf{v}_i}(S)^2 \mathbf{1}_{S}\mathbf{1}_{S}^{\top}} + \sum_{S \neq T}p_{\mathbf{v}_i}(S)p_{\mathbf{v}_i}(T)(\mathbf{1}_{T}\mathbf{1}_{S}^{\top} + \mathbf{1}_{S}\mathbf{1}_{T}^{\top}) \\
    &= \sum_{S,T \subseteq [n]}p_\mathbf{v_i}(S)p_\mathbf{v_i} (T)(\mathbf{1}_{S} \mathbf{1}_{T}^\top+ \mathbf{1}_{T} \mathbf{1}_{S}^\top)
    %\label{eq:crossterm}
\end{align*}
Summing over all eigenvectors $\mathbf{v}_i$ yields the relation $\mathbf{X}=\sum_{S,T \subseteq [n]}p_\mathbf{X}(S,T)(\mathbf{1}_{S} \mathbf{1}_{T}^\top+ \mathbf{1}_{T} \mathbf{1}_{S}^\top)$, proving dual SDP feasibility.

Next, consider an input $\mathbf{X} = \mathbf{1}_S \mathbf{1}_S^\top$. In this case, the only eigenvector is $\mathbf{1}_S $ with eigenvalue $\lambda=|S|$ since $ \mathbf{X}\mathbf{1}_S=\mathbf{1}_S (\mathbf{1}_S^\top \mathbf{1}_S) = \mathbf{1}_S |S| $. That is, $ p_\mathbf{X}(T', T) =  p_{\mathbf{1}_S}(T') p_{\mathbf{1}_S}(T)$.

For $\mathbf{X}=\mathbf{1}_S\mathbf{1}_S^\top$, $\mathbf{1}_S $ is clearly an eigenvector with eigenvalue $\lambda = |S|$
because $ \mathbf{X}\mathbf{1}_S=\mathbf{1}_S (\mathbf{1}_S^\top \mathbf{1}_S) = \mathbf{1}_S |S| $.
So, taking $\mathbf{\bar{1}}_S = \mathbf{1}_S / \sqrt{|S|}$ to be the normalized eigenvector of $\mathbf{X}$, we have $\mathbf{X} = |S|\mathbf{\bar{1}}_S\mathbf{\bar{1}}_S^\top = |S| \bigg (\frac{\mathbf{1}_S}{ \sqrt{|S|}} \bigg )\bigg (\frac{\mathbf{1}_S}{ \sqrt{|S|}}\bigg )^\top =p_\mathbf{X}(S,S)\mathbf{1}_S\mathbf{1}_S^\top$ for $ p_\mathbf{X}(S,S)=1$.
Therefore, the corresponding neural SFE is
\begin{align*}
    \barf(\mathbf{1}_S\mathbf{1}_S^\top)= p_\mathbf{X}(S,S)f(S \cap S) = f(S).
\end{align*}

All that remains is to show continuity of neural SFEs. Since the scalar SFE $p_\xb$ is continuous in $\xb$ by assumption, all that remains is to show that the map sending $\mathbf{X}$ to its eigenvector with $i$-th largest eigenvalue is continuous. We handle sign flip invariance of eignevectors by assuming a standard choice for eigenvector signs---e.g., by flipping the sign where necessary to ensure that the first non-zero coordinate is greater than zero.  The continuity of the mapping $\mathbf{X} \mapsto \mathbf{v}_i$  follows directly from Theorem 2 from  \cite{yu2015useful}, which is a variant of the Davis--Kahan theorem. The result shows that the angle between the $i$-th eigenspaces of two matrices $\mathbf{X}$ and $\mathbf{X}'$ goes to zero in the limit as  $\mathbf{X} \rightarrow \mathbf{X}'$. 
\end{proof}
\section{General Experimental Background Information}\label{app: general expt setup}

\subsection{Hardware and Software Setup}

All training runs were done on a single GPU at a time. Experiments were either run on 1) a server with 8 NVIDIA RTX 2080 Ti GPUs, or 2) 4 NVIDIA RTX 2080 Ti GPUs. All experiments are run using Python, specifically the PyTorch~\citep{paszke2019pytorch} framework (\href{https://github.com/pytorch/pytorch/blob/master/LICENSE}{see licence here}). For GNN specific functionality, such as graph data batching,  use the PyTorch Geometric (PyG)~\citep{fey2019fast} (MIT License).

We shall open source our code with MIT License, and have provided anonymized code as part of the supplementary material for reviewers. 

\subsection{Data Details}

This paper uses five graph datasets: ENZYMES, PROTEINS, IMDB-BINART, MUTAG, and COLLAB. All data is accessed via the standardized PyG API. In the case of COLLAB, which has 5000 samples available, we subsample the first 1000 graphs only for training efficiency. All experiments Use a  train/val/test split ratio of 60/30/10, which is done in exactly one consistent way across all experiments for each dataset.

\section{Unsupervised Neural Combinatorial Optimization Experiments}\label{app: unsup comb}
 All methods use the same GNN backbone: a combination of GAT \cite{velivckovic2018graph} and Gated Graph Convolution layer \citep{yujia2016gated}. We use the Adam optimizer \cite{kingma2014adam} with initial $lr=10^{-4}$ and default PyTorch settings for other parameters \cite{paszke2019pytorch}.   We use grid search HPO over batch size $\{4,32,64\}$, number of GNN layers $\{6,10,16\}$ network width $\{64,128,256\}$. All models are trained for $200$ epochs. For the model with the best validation performance, we report the test performance and the standard deviation of performance over test graphs as a measure of method reliability. 
 
\subsection{Discrete Objectives}\label{ap: discreteObjectives}

\paragraph{Maximum Clique.}

For the maximum clique problem, we could simply take $f$ to compute the clique size (with the size being zero if $S$ is not a clique). However, we found that this objective led to poor results and unstable training dynamics. So, instead, we select a discrete objective that yielded the much more stable results across datasets. It is defined for a graph $G = ([n], E)$ as,
\begin{align}
    f_{\text{MaxClique}}(S;G) = w(S)q^{c}(S), \label{eq:clique_objective}
\end{align}
where $w$ is a measure of size of $S$ and $q$ measures the density of edges within $S$ (i.e., distance from being a clique). The scalar $c$ is a constant, taken to be $c=2$ in all cases except REINFORCE for which $c=2$ proved ineffective, so we use $c=4$ instead. Specifically,  $w(S) = \sum_{i,j \in S}\mathbf{1}\{ (i,j) \in E\}$ simply counts up all the edges between nodes in $S$, and $q(S)=-2w(S)/(|S|^{2}-|S|)$ is the ratio (with a sign flip) between the number of edges in $S$, and the number of undirected edges $(|S|^{2}-|S|)/2$ there would be in a clique of size $|S|$. If $G$ were directed, simply remove the factor of $2$. Note that this $f$ is minimized when $S$ is a maximum clique. 

\paragraph{Maximum Independent Set.}

Similarly for maximum independent set we use the discrete objective,
\begin{align}
    f_{\text{MIS}}(S;G) = w(S)q^{c}(S), \label{eq:clique_objective}
\end{align}
where $w$ is a measure of size of $S$ and $q$ measures the number of edges  between nodes in $S$ (the number should be zero for an independent set), and $c=2$ as before. Specifically, we take  $w(S) = |S|/n$, and $q(s) = 2 \sum_{i,j \in S}\mathbf{1}\{ (i,j) \in E\} / (|S|^{2}-|S|)$, as before.

\subsection{Neural SFE details.}

All Neural SFEs, unless otherwise stated, use the top $k=4$ eigenvectors corresponding to the largest eigenvalues. This is an important efficiency saving step, since with $k=n$, i.e., using all eigenvectors, the resulting Neural Lova\'{s}z extension requires $O(n^2)$ set function evaluations, compared to $O(n)$ for  the scalar Lova\'{s}z  extension. By only using the top $k$ we reduce the number of evaluations to $O(kn)$. Wall clock runtime experiments given in Figure \ref{fig: runtime ablation} show that the runtime of the Neural Lova\'{s}z extension is around $\times k$ its scalar counterpart, and that the performance of the neural extension gradually increases then saturates when $k$ gets large. To minimize compute overheads we pick the smallest $k$ at which performance saturation approximately occurs. 

Instead of calling the pre-implemented PyTorch  eigensolver  \texttt{torch.linalg.eigh}, which  calls LAPACK routines, we use the power method to approximate the first $k$ eignevectors of $\mathbf{X}$. This is because we found the PyTorch function to be too numerically unstable in our case. In contrast, we found the power method, which approximates eigenvectors using simple recursively defined polynomials of $\mathbf{X}$, to be significantly more reliable. In all cases we run the power method for $5$ iterations, which we found to be sufficient for convergence.

\subsection{Baselines.}

This section discusses various implementation details of the baseline methods we used. The basic training pipeline is kept identical to SFEs, unless explicitly said otherwise. Namely, we use nearly identical model architectures, identical data loading, and identical HPO parameter grids.

\paragraph{REINFORCE.}
We compared with REINFORCE (\cite{williams1992simple}) which enables backpropagation through (discrete) black-box functions. We opt for a simple instantiation for the score estimator
\begin{align}
    \hat{g}_{\text{REINFORCE}} = f(S)\frac{\partial}{\partial \theta}\log p(S|\theta),
\end{align}
where $p(S|\theta)= \prod_{i \in S} p_i \prod_{j \notin S}(1-p_j) $, i.e., each node is selected independently with probability $p_i = g_{\theta}(\mathbf{y})$ for $i=1,2,\dots ,n$, where $g_\theta$ is a neural network and $\mathbf{y}$ some input attributes. We maximize the expected reward, i.e.,
\begin{align}
    L_{\text{REINFORCE}}(\theta) = \E_{S \sim \theta}[ \hat{g}_{\text{REINFORCE}}].
\end{align}
For all experiments with REINFORCE, the expected reward is computed over 250 sampled actions $S$ which is approximately the number of function evaluations of neural SFEs in most of the datasets.
Here, $f$ is taken to be the corresponding discrete objective of each problem (as described earlier in section \ref{ap: discreteObjectives}). For maximum clique, we normalize rewards $f(S)$ by removing the mean and dividing by the standard deviation. For the maximum independent set, the same strategy led to severe instability during training. To alleviate the issue, we introduced an additional modification to the rewards: among the sampled actions $S$, only the ones that achieved higher than average reward were retained and the rewards of the rest were set to 0. This led to more stable results in most datasets, with the exception of COLLAB were the trick was not sufficient. 

These issues highlight the instability of the score function estimator in this kind of setting. Additionally, we experimented by including simple control variates (baselines). These were: i) a simple greedy baseline obtained by running a greedy algorithm on each input graph ii) a simple uniform distribution baseline, where actions $S$ were sampled uniformly at random. Unfortunately, we were not able to obtain any consistent boost in either performance or stability using those techniques. Finally, to improve stability, the architectures employed with REINFORCE were slightly modified according to the problem. For example, for the independent set we additionally applied a sigmoid to the outputs of the final layer. 

\paragraph{Erdos Goes Neural.}We compare with recent work on unsupervised combinatorial optimization \citep{karalias2020erdos}. 
We use the probabilistic methodology described in the paper to obtain a loss function for each problem.
For the MaxClique, we use the loss provided in the paper, where for an input graph $G=([n],E)$ and learned probabilities $\mathbf{p}$ it is calculated by
\begin{align}
    L_{\text{Clique}}(\mathbf{p};G) = (\beta+1)\sum_{(i,j) \in E}{w_{ij}p_ip_j}  + \frac{\beta}{2}\sum_{v_i \neq v_j } {p_ip_j}.
\end{align} 
We omit additive constants as in practice they not affect the optimization. 
For the maximum independent set, we follow the methodology from the paper to derive the following loss:
\begin{align}
     L_{\text{IndepSet}}(\mathbf{p};G) =  \beta\sum_{(i,j) \in E}{w_{ij}p_ip_j}-\sum_{v_i \in V}p_i.
\end{align}
$\beta$ was tuned through a simple line search over a few possible values in each case. Following the implementation of the original paper, we use the same simple decoding algorithm to obtain a discrete solution from the learned probabilities.

\paragraph{Straight Through Estimator.}We also compared with the Straight-Through gradient estimator \citep{bengio2013estimating}. This estimator can be used to pass gradients through sampling and thresholding operations, by assuming in the backward pass that the operation is the identity. In order to obtain a working baseline with the straight-through estimator, we generate level sets according to the ranking of elements in the output vector $\xb$ of the neural network. Specifically, given $\xb\in [0,1]^{n}$ outputs from a neural network, we generate indicator vectors $\mathbf{1}_{S_{k}}$, where $S_{k} = \{j  | \; x_j \geq  x_k \}$ for $k=1,2,\dots,n$.
Then our loss function was computed as
\begin{align}
L_{ST}(\mathbf{x};G) =   \frac{1}{n}\sum_{k=1}^{n} f(\mathbf{1}_{S_{k}}),
\end{align}
where $f$ is the corresponding discrete objective from section \ref{ap: discreteObjectives}. At inference, we select the set that achieves the best value in the objective while complying with the constraints.

\paragraph{Ground truths.}  We obtain the maximum clique size and the maximum independent set size $s$ for each graph by expressing it as a mixed integer program and using the Gurobi solver \citep{gurobi}.

%  \begin{figure*}[t] %{6.5cm}
%   \begin{center}
%     \includegraphics[width=0.4\textwidth]{figs/score.pdf}
%       \vspace{-15pt}
%     \caption{Score function estimator reward over 100 epochs on maximum clique.}
%     \label{fig: k clique}
%   \end{center}
%   \vspace{-10pt}
% \end{figure*}

%  \begin{figure*}[t] %{6.5cm}
%   \begin{center}
%     \includegraphics[width=0.4\textwidth]{figs/lovasz_plot.pdf}
%       \vspace{-15pt}
%     \caption{Lova\'{s}z extension loss over 100 epochs.}
%     \label{fig: k clique}
%   \end{center}
%   \vspace{-10pt}
% \end{figure*}

 \begin{figure*}[t] %{6.5cm}
  \begin{center}
     \includegraphics[width=\textwidth]{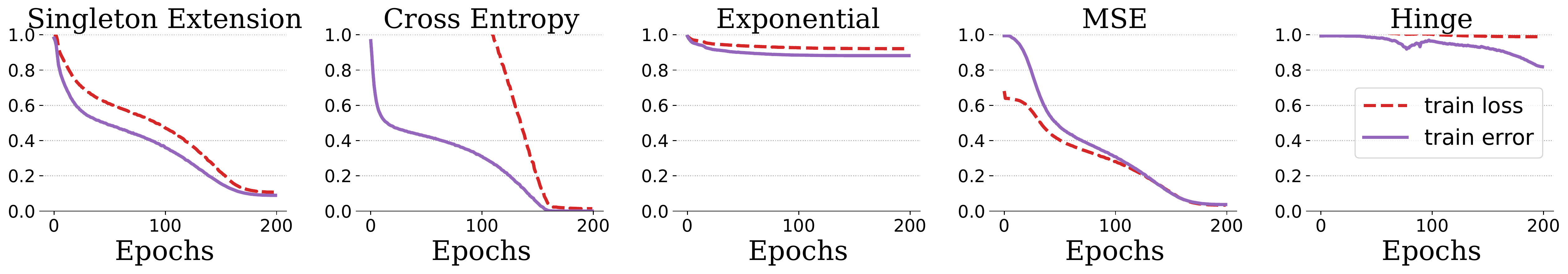}
        \includegraphics[width=\textwidth]{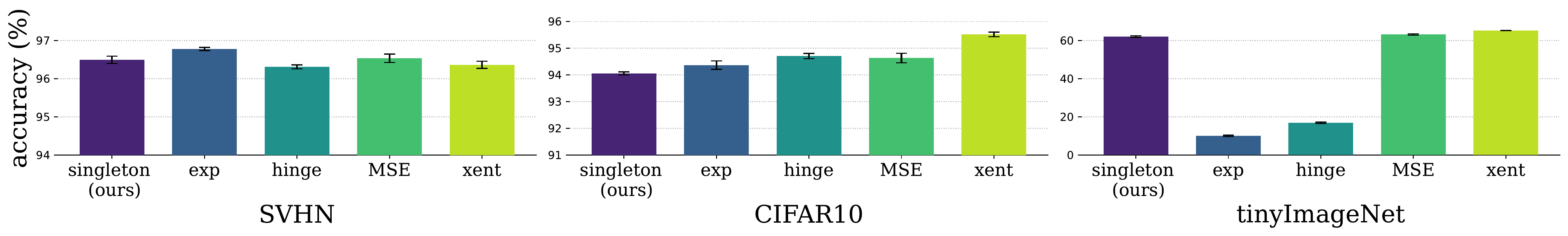}
      \vspace{-10pt}
    \caption{Top: Additional experimental results on the tinyImageNet dataset. Bottom: test accuracies of different losses. The singleton extension performs broadly comparably to other losses.}
    \label{fig: error v loss}
  \end{center}
  \vspace{-5pt}
\end{figure*}

\subsection{$k$-Clique Constraint Satisfaction}
\paragraph{Ground truths.}  As before, we obtain the maximum clique size $s$ for each graph by expressing it as a mixed integer program and using the Gurobi solver \citep{gurobi}. This is converted into a binary label $\mathbf{1}\{ s \geq k\}$ indicating if there is a clique of size $k$ or bigger. 

\paragraph{Implementation details.} The training pipeline, including HPO,  is identical to the MaxClique setup. The only difference comes in the evaluation---at test time the GNN produces an embedding $\xb$, and the largest clique $S$ in the support of $p_\xb$ is selected. The model prediction for the constraint satisfaction problem is then $\mathbf{1}\{ |S|\geq k\}$, indicating whether the GNN found a clique of size $k$ or more. Since this problem is. binary classification problem we compute the F1-score on a validation set, and report as the final result the F1-score of that same model on the   test set.

\section{Training error as an objective}\label{app: image classification}
Recall that for  a $K$-way classifier $h : \mathcal X \rightarrow \mathbb{R}^K$ with 
$\hat{y}(x) = \arg \max_{k =1,\ldots, K} h(x)_k$, we consider the training error  $\frac{1}{n}\sum_{i=1}^n \mathbf{1}\{ y_i \neq \hat{y}(x_i)\}$  calculated over a labeled training dataset $\{(x_i,y_i)\}_{i=1}^n$ to be a discrete non-differentiable loss. The set function in question is $y \mapsto \mathbf{1}\{ y_i \neq y\}$, which we relax using the singleton method described in Section \ref{sec: constructing scalar SFEs}.
\paragraph{Training details.} For all datasests we use a standard ResNet-18 backbone, with a final layer to output a vector of the correct dimension depending on the number of classes in the dataset.  CIFAR10 and tinyImageNet models are trained for 200 epochs, while SVHN uses 100 (which is sufficient for convergence). We use SGD with momentum $mom=0.9$ and weight decay $wd=5 \times 10^{-4}$ and a cosine learning rate schedule. We tune the learning rate for each loss via a simple grid search of  the values $lr \in \{ 0.01, 0.05, 0.1, 0.2\}$. For each loss we select the learning rate with highest accuracy on a validation set, then display the training loss and accuracy for this run.

%%%%%%%%%%%%%%%%%%%%%%%%%%%%%%%%%% PSEUDOCODE STUFF
\renewcommand{\algorithmicrequire}{\textbf{Input:}}
\renewcommand{\algorithmicensure}{\textbf{Output:}}

\definecolor{commentcolor}{RGB}{110,154,155}   % define comment color
\newcommand{\PyComment}[1]{\ttfamily\textcolor{commentcolor}{\# #1}}  % add a "#" before the input text "#1"
\newcommand{\codedef}[1]{\ttfamily\textcolor{blue}{#1}}  % add a "#" before the input text "#1"
\newcommand{\PyCode}[1]{\ttfamily\textcolor{black}{#1}} % \ttfamily is the code font

\section{Pseudocode: A forward pass of Scalar and Neural SFEs}
To illustrate the main conceptual steps in the implementation of SFEs, we include two torch-like pseudocode examples for SFEs, one for scalar and one for neural SFEs. The key to the practical implementation of SFEs within PyTorch is that it is only necessary to define the forward pass. Gradients are then handled automatically during the backwards pass. 

Observe that in both Algorithm, \ref{algo:scalarSFE} and Algorithm \ref{algo:NeuralSFE}, there are two key functions that have to be implemented: i) getSupportSets, which generates the sets on which the extension is supported.
    ii) getCoeffs, which generates the coefficients of each set.
Those depend on the choice of the extension and have to be implemented from scratch whenever a new extension is designed. The sets of the neural extension and their coefficients can be calculated from the corresponding scalar ones, using the definition of the Neural SFE and Proposition 3.
\begin{algorithm}[h]
\SetAlgoLined{}
    \codedef{def} ScalarSFE(setFunction, x): \\
    \Indp
    \PyComment{x: n x 1 tensor of embeddings, the output of a neural network \\
    \# n: number of items in ground set (e.g. number of nodes in graph) \\
    }
    \PyCode{setsScalar = getSupportSetsScalar(x)} \PyComment{n x n, i-th column is  $S_i$.} \\
    \PyCode{coeffsScalar = getCoeffsScalar(x)} \PyComment{1 x n: coefficients $y_{S_i}$.} \\
    \PyCode{extension = (coeffsScalar*setFunction(setsScalar)).sum()} \\
    \codedef{return} extension \\
    \Indm 
\caption{\textbf{Scalar} set function extension}
\label{algo:scalarSFE}
\end{algorithm}

\begin{algorithm}[h]
\SetAlgoLined{}
    \codedef{def} NeuralSFE(setFunction, X): \\
    \Indp  
    \PyComment{X: n x d tensor of embeddings, the output of a neural network \\
    \# n: number of items in ground set (e.g. number of nodes in graph) \\
    \# d: embedding dimension} \\
    \PyCode{X = normalize(X, dim=1)} \\
    \PyCode{Gram = X @ X.T} \PyComment{ n x n} \\ 
    \PyCode{eigenvalues, eigenvectors = powerMethod(Gram)} \\
    \PyCode{extension = 0}  \PyComment{initialize variable} \\
    \PyCode{for (eigval,eigvec) in zip(eigenvalues,eigenvectors):} \\
\Indp   % start indent
    \PyComment{Compute scalar extension data.} \\
           \PyCode{setsScalar = getSupportSetsScalar(eigvec)} \\
               \PyCode{coeffsScalar = getCoeffsScalar(eigvec)}  \\
               \PyComment{Compute neural extension data from scalar extension data.} \\
           \PyCode{setsNeural = getSupportSetsNeural(setsScalar)} \\
    \PyCode{coeffsNeural = getCoeffsNeural(coeffsScalar)} \\
    \PyCode{extension += eigval*((coeffsNeural*setFunction(setsNeural)).sum())} \\ 
        \Indm 
    \codedef{return} extension\\
    \Indm 
\caption{\textbf{Neural} set function extension}
\label{algo:NeuralSFE}
\end{algorithm}

\section{Further Discussion}

\subsection{Limitations and Future Directions}\label{sec: limitations}

Our SFEs have proven useful for learning solvers for a number of combinatorial optimization problems. However there remain many directions for improvement. One direction of particular interest is to scale our methods to instances with very large $n$. This could include simply considering larger graphs, or problems with larger ground sets---e.g., selecting paths. We believe that a promising approach to this would be to develop localized extensions that are supported on sets corresponding to suitably chosen sub-graphs, which would enable us to build in additional task-specific information about the problem.

\subsection{Broader Impact}\label{sec: broader impact}

Our work focuses on a core machine learning methodological goal of designing neural networks that are able to learn to simulate algorithmic behavior. This program may lead to a number of promising improvements in neural networks such as making their generalization properties more reliable (as with classical algorithms) and more interpretable decision making mechanisms. As well as injecting algorithmic properties into neural network models, our work studies the use of neural networks for solving combinatorial problems. Advances in neural network methods may lead to advances in numerical computing more widely.   Numerical computing in general---and combinatorial optimization in particular---impacts a wide range of human activities, including scientific discovery and logistics planning. Because of this, the methodologies developed in this paper and any potential further developments in this line of work are intrinsically neutral with respect to ethical considerations; the main responsibility lies in their ethical application in any given scenario.

%%%%%%%%%%%%%%%%%%%%%%%%%%%%%%%%%%%%%%%%%%%%%%%%%%%%%%%%%%%%%%%%%%%%%%%%%%%%%%%%%

\end{document}